\documentclass{article}
\usepackage[square,sort,comma,numbers]{natbib}
\usepackage[preprint]{neurips_2026}

\usepackage{microtype}
\usepackage{hyperref}
\usepackage{url}
\usepackage{booktabs}
\usepackage{amsmath}

\usepackage{xspace}
\usepackage[utf8]{inputenc}
\usepackage[T1]{fontenc}
\usepackage{hyperref}
\usepackage{url}
\usepackage{booktabs}
\usepackage{amsfonts}
\usepackage{amsmath}
\usepackage{amssymb}
\usepackage{amsthm}
\usepackage{nicefrac}
\usepackage{microtype}
\usepackage{xcolor}
\usepackage{graphicx}
\usepackage{algorithm}
\usepackage{algorithmic}
\usepackage{multirow}
\usepackage{subcaption}
\usepackage{enumitem}
\usepackage{mathtools}
\usepackage{bbm}
\usepackage{dsfont}
\usepackage[table]{xcolor}

\newtheorem{theorem}{Theorem}

\newtheorem{corollary}[theorem]{Corollary}
\newtheorem{proposition}{Proposition}
\newtheorem{assumption}{Assumption}

\newtheorem{remark}{Remark}

\newcommand*{\inlineequation}[2][]{%
  \begingroup
    \refstepcounter{equation}%
    \ifx\\#1\\%
    \else
      \label{#1}%
    \fi
    \relpenalty=10000 %
    \binoppenalty=10000 %
    \ensuremath{%
      #2%
    }%
    ~\@eqnnum
  \endgroup
}
\usepackage{lineno}

\definecolor{darkblue}{rgb}{0, 0, 0.5}
\hypersetup{colorlinks=true, citecolor=darkblue, linkcolor=darkblue, urlcolor=darkblue}

\newcommand{\cg}{\textsc{CausalGuard}\xspace}

\title{\cg{}: Conformal Inference under Graph Uncertainty}

\author{%
  \textbf{Vikash Singh\textsuperscript{1}} \quad
  \textbf{Weicong Chen\textsuperscript{1}} \quad
  \textbf{Debargha Ganguly\textsuperscript{1}} \quad
  \textbf{Yanyan Zhang\textsuperscript{1}} \\
  \textbf{Nengbo Wang\textsuperscript{1}} \quad
  \textbf{Sreehari Sankar\textsuperscript{1}} \quad
  \textbf{Mohsen Hariri\textsuperscript{1}} \quad
  \textbf{Alexander Nemecek\textsuperscript{1}} \\
  \textbf{Chaoda Song\textsuperscript{1}} \quad
  \textbf{Shouren Wang\textsuperscript{1}} \quad
  \textbf{Biyao Zhang\textsuperscript{1}} \quad
  \textbf{Van Yang\textsuperscript{1}} \\
  \textbf{Erman Ayday\textsuperscript{1}} \quad
  \textbf{Jing Ma\textsuperscript{1}} \quad
  \textbf{Vipin Chaudhary\textsuperscript{1}} \\[0.3em]
  \textsuperscript{1}Case Western Reserve University \\[0.3em]
  \texttt{vikash@case.edu}
}

\begin{document}

\maketitle
\begin{abstract}
Estimating treatment effects from observational data requires choosing an adjustment set, but valid adjustment depends on an unknown causal graph. Graph misspecification can cause under-coverage, while graph-agnostic conformal wrappers may regain nominal coverage only through large padding. We introduce \cg{}, a structure-weighted conformal framework that calibrates after aggregating graph-conditional doubly robust pseudo-outcomes. Candidate DAGs are proposed from an LLM-derived edge prior, pruned by conditional-independence tests, and reweighted by Bayesian Information Criterion. A composite nonconformity score then calibrates the posterior-weighted pseudo-outcome. \cg{} provides distribution-free finite-sample marginal coverage for this aggregated pseudo-outcome; under causal identification, overlap, conditional-mean nuisance stability, and concentration on target-aligned valid adjustment strategies, its conditional mean converges to the true Conditional Average Treatment Effect. Across five benchmarks, \cg{} attains mean coverage above the nominal 90\% level for the directly evaluable target and reduces width when graph-agnostic conformal baselines require large padding. Stress tests show that \cg{} suppresses invalid collider adjustment and remains stable under misspecified priors when the retained candidate set is data-supported.

\end{abstract}

\section{Introduction}

\label{sec:intro}

Estimating individualized treatment effects from observational data requires choosing an adjustment set that blocks confounding. That choice depends on the underlying causal Directed Acyclic Graph (DAG), which is rarely known in practice~\citep{pearl2009causality}. Modern heterogeneous-effect estimators such as doubly robust learners~\citep{chernozhukov2018double} and causal forests~\citep{wager2018estimation} can be statistically efficient once a valid graph is specified, but they are brittle to structural misspecification: omitting confounders destroys identification, while adjusting for colliders or treatment descendants can induce additional bias~\citep{spirtes2000causation}.

Conformal inference is a natural response to this uncertainty because it provides finite-sample marginal coverage without distributional assumptions~\citep{lei2021conformal,romano2019conformalized}. However, under causal graph uncertainty, graph-agnostic conformal wrappers can preserve marginal coverage for their observable calibration target; but if that target is based on an invalid adjustment set, the resulting interval need not provide formal coverage for the causal estimand.

This work addresses that limitation through \emph{structure-weighted conformal inference}, instantiated in our method \cg{}. Instead of committing to a single graph, \cg{} constructs an ensemble of plausible DAGs, fits graph-conditional doubly robust estimators under each corresponding adjustment set, and aggregates and conformalizes the resulting predictions. The central object is a composite nonconformity score defined on a posterior-weighted pseudo-outcome, which transfers structural uncertainty directly into the conformal calibration step.

The resulting finite-sample guarantee applies to a structure-weighted pseudo-outcome, not directly to the true CATE. Under causal identification and overlap, conditional-mean nuisance stability, and score concentration on target-aligned valid adjustment strategies, the conditional mean of this target converges to the true CATE. In the implemented framework, candidate graphs are proposed by an LLM-derived edge prior, pruned by conditional-independence tests, and reweighted by Bayesian Information Criterion (BIC)~\citep{schwarz1978estimating}. 

This formulation also resolves a key practical tension. In finite samples, a semantic prior can regularize the graph search toward plausible adjustment sets. As the sample size increases, the data likelihood can dominate and the prior can wash out. The method is designed to exhibit exactly this behavior: the prior organizes finite-sample graph search while the final interval becomes increasingly data-driven when the likelihood is informative.

We summarize the main contributions as follows:

\begin{itemize}[nosep,leftmargin=1.3em]

    \item \textbf{Structure-weighted conformal inference.} A conformal construction for treatment-effect intervals under graph uncertainty that aggregates graph-conditional doubly robust pseudo-outcomes through a posterior-weighted composite score.

    \item \textbf{Aggregate-before-calibrate scoring.} A composite score that conformalizes the posterior-weighted target directly, avoiding the Jensen gap that arises when graph-specific conformity scores or intervals are averaged after calibration.

    \item \textbf{Graph-ensemble construction.} A practical procedure combining LLM-derived edge proposals, conditional-independence pruning, and BIC reweighting, in which the prior acts as a proposal distribution rather than a source of guaranteed causal truth.

    \item \textbf{Finite-sample validity and asymptotic target consistency.} Distribution-free finite-sample marginal coverage for the weighted pseudo-outcome (Theorem~\ref{thm:coverage}) and conditional-mean convergence of this target to the true CATE under causal identification, overlap, conditional-mean nuisance stability, and target-aligned score-concentration conditions (Corollary~\ref{cor:target_consistency}).

    \item \textbf{Empirical validation and boundary analysis.} Across five benchmarks and 10 random seeds per dataset, \cg attains mean coverage above the nominal 90\% target for the directly evaluable target. Stress tests show that graph-aware adjustment suppresses invalid collider adjustment, while low-data diagnostics show when graph-aware conformalization is substantially tighter than graph-agnostic conformal DR.

\end{itemize}

\section{Related Work}
\label{sec:related_work}


\textbf{Conformal inference for treatment effects.}
Conformal prediction provides finite-sample marginal coverage with minimal distributional assumptions~\citep{vovk2005algorithmic,lei2018distribution}. Recent causal extensions include conformal intervals for counterfactuals and individualized treatment effects~\citep{lei2021conformal,alaa2023conformal}, conformalized quantile regression~\citep{romano2019conformalized}, covariate-shift adjustments~\citep{tibshirani2019conformal}, and sensitivity analyses for hidden confounding~\citep{jin2023sensitivity}. These methods assume the adjustment set is fixed before calibration. Under graph uncertainty, both the pseudo-outcome and the conformity-score distribution depend on the graph. Running conformal calibration separately for each candidate graph and averaging afterward targets a different, typically wider object. \cg{} instead aggregates graph-conditional pseudo-outcomes and bounds first, then calibrates the resulting composite score.

\textbf{Causal estimation with uncertain structure.}
Standard heterogeneous-effect estimators such as Double Machine Learning~\citep{chernozhukov2018double} and Causal Forests~\citep{wager2018estimation} achieve strong asymptotic performance when a valid adjustment set is known. Under graph misspecification, colliders or treatment descendants can be inadvertently included in the adjustment set, degrading finite-sample performance~\citep{pearl2009causality}. \cg{} instead treats the adjustment set as uncertain and propagates that uncertainty through the interval construction.
Related work on covariate and adjustment-set selection studies how to choose valid or efficient adjustment variables from graphical or statistical criteria~\citep{entner2013data,henckel2022graphical}. That line of work typically returns a single selected adjustment strategy. Our focus is complementary: when several adjustment strategies remain plausible, we carry their uncertainty through doubly robust pseudo-outcomes and conformal calibration rather than collapsing it before inference.

\textbf{Causal discovery and LLM-derived priors.}
Classical discovery methods such as PC, GES, and continuous DAG optimization estimate a DAG or Markov equivalence class from observational data alone~\citep{spirtes2000causation,chickering2002optimal,zheng2018dags}. Recent work studies whether LLMs can contribute causal priors, edge orientations, or graph refinement heuristics~\citep{kiciman2023causal,long2023causal,ban2023query,jiralerspong2024efficient,sgouritsa2024prompting}. These works primarily evaluate graph recovery or causal-reasoning accuracy. More broadly, uncertainty-aware and verification-guided LLM systems show that model confidence should be treated as task-dependent evidence rather than as an oracle signal~\citep{ganguly2025grammars,singh2026verge}. \cg{} uses LLM-derived edge probabilities more conservatively: they define a proposal distribution over plausible DAGs, and their value is assessed only through downstream interval estimation after BIC reweighting and conformal calibration.

\textbf{Model averaging and graph ensembles.}
Bayesian model averaging hedges against model uncertainty by integrating predictions over a posterior distribution on candidate models~\citep{madigan1994model,hoeting1999bayesian}. Related ideas appear in causal inference through Bayesian treatment-effect models and effect aggregation across equivalence classes~\citep{hill2011bayesian,maathuis2009estimating}. These approaches typically provide posterior or asymptotic uncertainty statements. In contrast, \cg{} uses graph weights to define the object that is subsequently calibrated by split conformal inference, yielding distribution-free finite-sample marginal coverage for the aggregated target under exchangeability. The key distinction is therefore the calibration target: \cg{} first induces a structure-weighted pseudo-outcome through graph averaging, and then applies distribution-free calibration to that aggregated target.

\section{Methodology}
\label{sec:methodology}

An overview of the \cg{} pipeline is shown in Figure~\ref{fig:architecture}. The method proceeds in four stages: structural proposal, graph weighting, graph-conditional estimation, and structure-weighted conformal calibration. The pseudocode for the full procedure is provided in Appendix~\ref{app:algorithm}.

\begin{figure}[htbp]
    \centering
    \includegraphics[width=\textwidth]{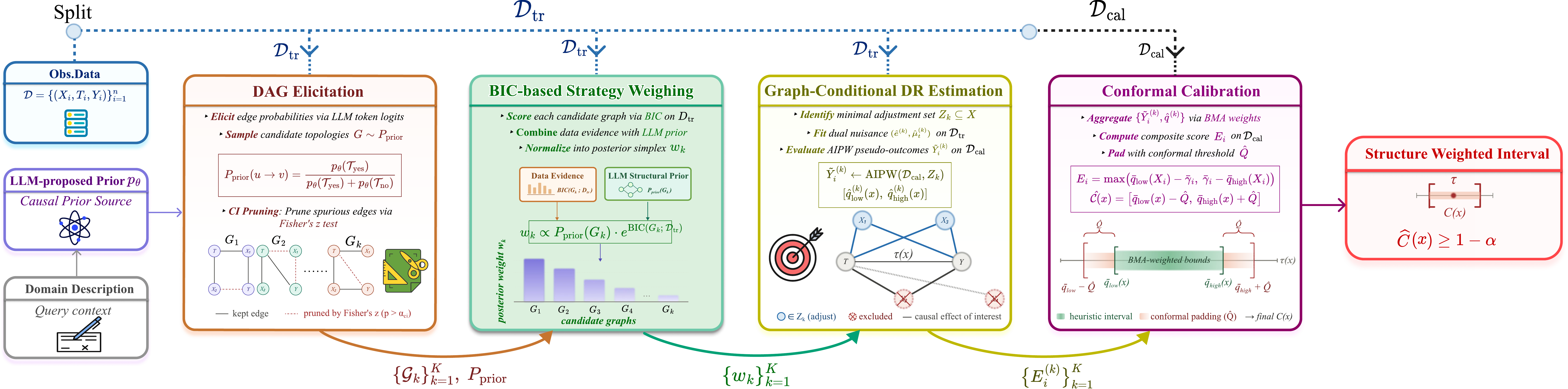}
    \caption{\textbf{The \cg pipeline.} Candidate DAGs are proposed from an LLM-derived edge prior, pruned by CI tests on $\mathcal{D}_{\rm train}$, scored by BIC plus a bounded structural prior, and collapsed to unique adjustment strategies before weight normalization. Graph-conditional doubly robust pseudo-outcomes are then aggregated into $\bar{\gamma}$, and a composite split-conformal score calibrated on $\mathcal{D}_{\rm cal}$ produces $\widehat C(x)$. The distribution-free finite-sample guarantee is for the structure-weighted pseudo-outcome $\bar{\gamma}$; the connection to $\tau(x)$ is through conditional-mean target consistency under additional assumptions.
    }
    \label{fig:architecture}
\end{figure}

The analysis is conducted in the Neyman--Rubin potential-outcomes framework~\citep{neyman1923application,rubin1974estimating}. Let $X=(X_1,\ldots,X_d)\in\mathbb{R}^d$ denote observed pre-treatment covariates, $T\in\{0,1\}$ a binary treatment, and $Y\in\mathbb{R}$ the observed outcome. Under the Stable Unit Treatment Value Assumption~\citep{rubin1980randomization}, each unit has potential outcomes $Y(1)$ and $Y(0)$, and the observed outcome satisfies $Y = T Y(1) + (1-T)Y(0)$. The target estimand is the Conditional Average Treatment Effect (CATE),
\begin{equation}
  \label{eq:cate}
  \tau(x) \;=\; \mathbb{E}\bigl[Y(1) - Y(0) \mid X = x\bigr].
\end{equation}
The goal is to construct an interval $\hat{C}(x) = [\hat{L}(x), \hat{U}(x)]$ that remains informative even when the true causal DAG $\mathcal{G}^*$ is unknown. The challenge is that identification of $\tau(x)$ depends on choosing a valid adjustment set, and that choice is itself graph dependent.

Let $V = \{T, Y, X_1, \ldots, X_d\}$ denote the full variable set. Each candidate DAG is written as $G_k = (V, E_k)$, where $E_k$ is its directed edge set; for any graph $G$ on $V$, $\mathrm{Pa}_{G}(v)$ denotes the parent set of node $v$. The observed sample is split into disjoint training and calibration subsets, $\mathcal{D}_{\text{train}}$ and $\mathcal{D}_{\text{cal}}$. All graph weights, nuisance models, and preliminary bounds are estimated on $\mathcal{D}_{\text{train}}$, while conformal scores are evaluated on $\mathcal{D}_{\text{cal}}$.

\subsection{Probabilistic Structural Prior Elicitation.}
\label{sec:llm_priors}
Because observational causal discovery is only partially identifiable without strong assumptions~\citep{spirtes2000causation,pearl2009causality}, a proposal distribution over edges is constructed. For each ordered pair $(u,v) \in V \times V$, a pre-trained autoregressive language model $p_\theta$ is queried about the existence of the directed edge $u \to v$. Let $\mathcal{T}_{\mathrm{yes}}$ and $\mathcal{T}_{\mathrm{no}}$ denote token sets corresponding to positive and negative answers, and let $\text{query}_{u,v}$ denote the prompt asking whether $u$ is a direct cause of $v$. The prior edge probability is defined as
\begin{equation}
  \label{eq:logprob}
  P_{\text{prior}}(u \to v)
  \;=\;
  \frac{
    \sum_{t \in \mathcal{T}_{\text{yes}}} p_\theta(t \mid \text{query}_{u,v})
  }{
    \sum_{t \in \mathcal{T}_{\text{yes}}} p_\theta(t \mid \text{query}_{u,v})
    \;+
    \sum_{t \in \mathcal{T}_{\text{no}}} p_\theta(t \mid \text{query}_{u,v})
  }.
\end{equation}
All probabilities are clipped to $[\epsilon, 1-\epsilon]$ for a small constant $\epsilon > 0$ to ensure numerical stability. Importantly, $P_{\text{prior}}$ serves only as a proposal mechanism over candidate graphs and does not define the inferential target.

\subsection{DAG Sampling, Statistical Pruning, and Bayesian Weighting.}
\label{sec:dag_sampling}


Given $P_{\text{prior}}$, a finite ensemble of $K$ candidate DAGs is instantiated. For each graph $G_k$, a topological order $\pi_k$ over $V$ is first sampled subject to temporal metadata. By default, benchmark covariates are treated as baseline variables satisfying $X_j \prec T \prec Y$ for all $j$; variables marked as post-treatment are ordered after treatment and are excluded from admissible adjustment sets. Edges among variables with the same temporal status remain ordered according to $\pi_k$. Each admissible forward edge $u \to v$ with $\pi_k(u) < \pi_k(v)$ is then sampled independently from a Bernoulli distribution with parameter $P_{\text{prior}}(u \to v)$, which guarantees acyclicity by construction~\citep{friedman2003being}.

Edges are subsequently pruned using conditional-independence checks. For each sampled graph $G_k$ and edge $u \to v \in E_k$, let $S = \mathrm{Pa}_{G_k}(v) \setminus \{u\}$ denote the remaining parents of $v$. The edge is tested for conditional independence using a standard partial-correlation test (based on the Fisher-transformed statistic~\citep{fisher1921probable}) at significance level $\alpha_{\mathrm{CI}}$, and edges for which independence is not rejected are removed. This step acts only as a proposal-stage filter prior to weighting, not as a standalone discovery procedure.

Residual structural uncertainty is represented through posterior-like weights over the pruned candidate graphs, followed by a collapse to unique adjustment strategies. Let $\mathcal{A}_k$ denote the set of admissible directed pairs induced by this temporally constrained order $\pi_k$, i.e., all $(u,v)$ such that $\pi_k(u) < \pi_k(v)$ and the temporal constraints above are respected, excluding self-loops. For each candidate $G_k$, a Bayesian Information Criterion (BIC) score~\citep{schwarz1978estimating} is combined with the edge-wise Bernoulli prior:
\begin{equation}
  \label{eq:bma_weight}
  \ln \tilde{w}_k
  \;=\;
  \underbrace{\mathrm{BIC}(G_k \mid \mathcal{D}_{\text{train}})}_{\text{Data Evidence}}
  +
  \underbrace{
    \textstyle\sum_{(u,v)\in E_k} \ln P_{\text{prior}}(u \to v)
    +
    \textstyle\sum_{(u,v) \in \mathcal{A}_k \setminus E_k} \ln \bigl(1 - P_{\text{prior}}(u \to v)\bigr)
  }_{\text{Bernoulli Structural Prior}}.
\end{equation}
After CI pruning and adjustment-set extraction, duplicate selected adjustment strategies are collapsed before normalization. Let $\mathcal{Z}$ denote the set of unique selected adjustment sets among surviving graphs. For each $Z\in\mathcal{Z}$,
\begin{equation}
  \label{eq:adjustment_collapse}
  \tilde{w}_{Z}
  =
  \textstyle\sum_{k: Z_k=Z}\exp(\ln \tilde{w}_k),
  \qquad
  w_Z
  =
  \tilde{w}_{Z}\big/\sum_{Z'\in\mathcal{Z}}\tilde{w}_{Z'} .
\end{equation}

\subsection{Graph-Conditional Doubly Robust Estimation.}
\label{sec:dr_estimation}
For each graph $G_k$, let $Z_k \subseteq \{X_1,\ldots,X_d\}$ denote a backdoor adjustment set implied by $G_k$, i.e., valid under $G_k$ via the backdoor criterion~\citep{pearl2009causality}. When multiple valid sets are returned, the implementation uses a minimum-cardinality valid set; graphs for which no admissible adjustment set is found are excluded before duplicate-strategy collapse and weight normalization. Standard doubly robust nuisance models are fitted on $\mathcal{D}_{\text{train}}$ to obtain the Augmented Inverse Probability Weighting (AIPW) pseudo-outcomes~\citep{robins1994estimation, chernozhukov2018double} and graph-conditional heuristic bounds $[\hat{q}_{\text{low}}^{(k)}, \hat{q}_{\text{high}}^{(k)}]$, interpreted as lower and upper prediction functions, not necessarily valid confidence bounds, for those pseudo-outcomes. For a unit with covariates $X_i$, write $Z_{i,k}$ for the subvector of $X_i$ restricted to the variables in $Z_k$. The precise $\mathcal{D}_{\text{train}}$ / $\mathcal{D}_{\text{cal}}$ split, nuisance-fitting details, and implementation choices are deferred to Appendix~\ref{app:data_sep} and Appendix~\ref{app:adjustment_sets}.

For calibration, the AIPW pseudo-outcome is evaluated on $\mathcal{D}_{\text{cal}}$:
\begin{equation}
  \label{eq:aipw}
  \tilde{Y}_i^{(k)}
  \;=\;
  \hat{\mu}_1^{(k)}(Z_{i,k}) - \hat{\mu}_0^{(k)}(Z_{i,k})
  \;+
  \frac{T_i\bigl(Y_i - \hat{\mu}_1^{(k)}(Z_{i,k})\bigr)}{\hat{e}^{(k)}(Z_{i,k})}
  \;-
  \frac{(1 - T_i)\bigl(Y_i - \hat{\mu}_0^{(k)}(Z_{i,k})\bigr)}{1 - \hat{e}^{(k)}(Z_{i,k})}.
\end{equation}
Here $\hat{e}^{(k)}$ is the propensity score and $\hat{\mu}_t^{(k)}$ are outcome regressions. This stage produces graph-conditional pseudo-outcomes that can be aggregated under structural uncertainty.

\subsection{Structure-Weighted Conformal Calibration.}
\label{sec:conformal}
The conformal step is where graph uncertainty enters the interval construction. Because all graph weights, nuisance models, and heuristic bounds are computed using $\mathcal{D}_{\text{train}}$ only, the calibration scores and the future test score are exchangeable conditional on $\mathcal{D}_{\text{train}}$.

Conceptually, the pipeline separates two roles. Structure weighting attempts to reduce bias from uncertain adjustment choices by averaging over data-supported adjustment strategies before calibration, while conformal calibration accounts for the remaining finite-sample uncertainty in the aggregated estimator. This separation is useful when the retained ensemble contains target-aligned valid adjustment strategies; otherwise, the conformal guarantee still applies to the aggregated pseudo-outcome but does not by itself repair causal misspecification.

For each calibration point $i$ and graph $k$, define the graph-conditional nonconformity score
\begin{equation}
  \label{eq:score}
  E_i^{(k)}
  \;=\;
  \max\;\Bigl(
    \hat{q}_{\text{low}}^{(k)}(X_i) - \tilde{Y}_i^{(k)},\;\;
    \tilde{Y}_i^{(k)} - \hat{q}_{\text{high}}^{(k)}(X_i)
  \Bigr).
\end{equation}

This graph-specific score is introduced for exposition; the final conformal calibration is performed using the aggregated score defined below.
Under graph uncertainty, no single adjustment set is privileged \emph{a priori}. The Bayesian Model Averaging (BMA)-weighted pseudo-outcome~\citep{hoeting1999bayesian} therefore defines the natural statistical object induced by the posterior over admissible adjustment strategies. Rather than selecting a single graph and treating alternatives as noise, the method integrates over graph-conditioned pseudo-outcomes supported by both prior knowledge and observed data.
Equivalently, the posterior-weighted pseudo-outcome $\bar{\gamma}$ (defined below) can be interpreted as a surrogate estimand under structural uncertainty, analogous to BMA over adjustment sets. This aggregated object serves as the finite-sample target of the conformal guarantee, while $\tau$ remains the asymptotic causal target.

Define the aggregated pseudo-outcome and aggregated bounds as
\[
\bar{\gamma}_i = \textstyle\sum_{k=1}^K w_k \tilde{Y}_i^{(k)},\qquad
\bar{q}_{\text{low}}(X_i) = \textstyle\sum_k w_k \hat{q}_{\text{low}}^{(k)}(X_i),\qquad
\bar{q}_{\text{high}}(X_i) = \textstyle\sum_k w_k \hat{q}_{\text{high}}^{(k)}(X_i).
\]

This order of operations is central. BMA alone gives a posterior-weighted prediction without finite-sample frequentist validity, while conformalizing a single graph-conditioned estimator ignores adjustment-set uncertainty. \cg{} instead uses graph weighting to define the calibration target and split conformal inference to calibrate that target, changing the object being calibrated under graph uncertainty.
The composite nonconformity score is
\begin{equation}
\label{eq:composite_score}
  E_i
  \;=\;
  \max\;\Bigl(
    \bar{q}_{\text{low}}(X_i) - \bar{\gamma}_i,\;\;
    \bar{\gamma}_i - \bar{q}_{\text{high}}(X_i)
  \Bigr).
\end{equation}

\paragraph{Why aggregate before calibration?}
The following proposition formalizes the key efficiency rationale for directly conformalizing the posterior-weighted target.

\begin{proposition}[Aggregate-before-calibrate domination]
\label{prop:aggregate_before_calibrate}
Fix the training split, graph weights, graph-conditional pseudo-outcomes, and graph-conditional bounds. Define the graph-wise averaged score
\begin{equation}
\label{eq:jensen_score}
E_i^{\dagger}
=
\textstyle\sum_{k=1}^K w_k
\max\!\left(
\hat q_{\mathrm{low}}^{(k)}(X_i)-\tilde Y_i^{(k)},
\tilde Y_i^{(k)}-\hat q_{\mathrm{high}}^{(k)}(X_i)
\right).
\end{equation}
Then $E_i\le E_i^\dagger$ for every calibration point $i$. Hence the split-conformal quantile based on the aggregated score $\{E_i\}$ is no larger than the quantile based on $\{E_i^\dagger\}$ on the same calibration set using the same quantile rule.
\end{proposition}

\begin{proof}
Let $a_{ik}=\hat q_{\mathrm{low}}^{(k)}(X_i)-\tilde Y_i^{(k)}$,
$b_{ik}=\tilde Y_i^{(k)}-\hat q_{\mathrm{high}}^{(k)}(X_i)$; convexity of
$(a,b)\mapsto\max(a,b)$ gives
\begin{equation}
\label{eq:jensen_gap}
E_i=\max\!\left(\textstyle\sum_k w_k a_{ik},\textstyle\sum_k w_k b_{ik}\right)
\le \textstyle\sum_k w_k\max(a_{ik},b_{ik})=E_i^\dagger .
\end{equation}
The quantile statement follows from the monotonicity of empirical order statistics under pointwise domination.
\end{proof}

Proposition~\ref{prop:aggregate_before_calibrate} shows that the method does not simply average already padded graph-specific intervals. It calibrates the posterior-weighted target itself, avoiding the extra Jensen gap introduced by graph-wise calibration followed by averaging.

Let $\hat{Q}$ be the $\lceil (1-\alpha)(|\mathcal{D}_{\text{cal}}|+1) \rceil$-th smallest value among the calibration scores $\{E_i\}_{i \in \mathcal{D}_{\text{cal}}}$. The final interval for a test point $x$ is
\begin{equation}
  \label{eq:final_interval}
  \hat{C}(x)
  \;=\;
  \biggl[
    \bar{q}_{\text{low}}(x) - \hat{Q},
    \;\;
    \bar{q}_{\text{high}}(x) + \hat{Q}
  \biggr].
\end{equation}

\subsection{Uncertainty Decomposition and Theoretical Guarantees.}
\label{sec:decomposition}
The resulting interval reflects three sources of uncertainty: model uncertainty in the nuisance estimators, finite-sample calibration uncertainty through $\hat{Q}$, and structural uncertainty across graphs. Let $\hat{\tau}_k(x)$ denote the graph-conditional CATE point estimate under $G_k$. Structural uncertainty can then be quantified via the posterior dispersion of graph-conditional predictions,
\begin{equation}
  \label{eq:struct_unc}
  \sigma_{\text{struct}}(x)
  \;=\;
  \sqrt{\textstyle\sum_{k=1}^K w_k \bigl(\hat{\tau}_k(x) - \bar{\tau}(x)\bigr)^2},
\end{equation}
where $\bar{\tau}(x) = \sum_k w_k \hat{\tau}_k(x)$.

This quantity should be interpreted as estimator-level graph dispersion, not edge-level graph uncertainty. Distinct sampled DAGs may imply nearly identical adjustment sets or nearly identical AIPW pseudo-outcomes, in which case $\sigma_{\text{struct}}(x)$ can be small even when the graph ensemble is visually diverse. This distinction matters empirically: \cg is most useful when graph uncertainty propagates into materially different pseudo-outcomes before calibration, rather than when many graphs induce the same effective adjustment strategy.

The distribution-free finite-sample guarantee concerns the BMA-weighted pseudo-outcome $\bar{\gamma}_{n+1}$, not the unobserved CATE itself. This pseudo-outcome remains noisy even when the nuisance functions are correct. The causal object recovered asymptotically is its conditional mean,
\begin{equation}
\label{eq:method_target_mean}
m_n(x)
=
\mathbb{E}\!\left[\bar{\gamma}_{n+1}\mid X_{n+1}=x,\mathcal{D}_{\mathrm{train}}\right].
\end{equation}

Under exchangeability, Theorem~\ref{thm:coverage} establishes finite-sample marginal coverage $\mathbb{P}(\bar{\gamma}_{n+1} \in \hat{C}(X_{n+1})) \ge 1-\alpha$. Corollary~\ref{cor:target_consistency} then shows that $m_n$ converges to $\tau$ in $L_1(P_X)$ under causal identification and overlap, conditional-mean nuisance stability, and score concentration on target-aligned valid adjustment strategies. Thus, the interval is finite-sample valid for the aggregated pseudo-outcome, while the conditional mean of that target aligns asymptotically with the CATE under these assumptions. 
Appendix~\ref{app:coverage_targets} clarifies the distinction between finite-sample pseudo-outcome coverage and CATE diagnostics.

\section{Empirical Evaluation}
\label{sec:experiments}

The empirical study addresses three questions: \textbf{(RQ1)} whether graph-aware conformalization improves efficiency, \textbf{(RQ2)} whether the method behaves correctly when structure matters, and \textbf{(RQ3)} when the prior matters or washes out. On IHDP and Twins, where ground-truth $\tau$ is available, coverage is reported for the true CATE. On Sachs, Jobs, and ACIC2019, coverage is reported for the weighted pseudo-outcome, the distribution-free finite-sample target of Theorem~\ref{thm:coverage}; Appendix~\ref{app:coverage_targets} details this distinction.

\subsection{Setup and Main Benchmark Results}

\cg is evaluated on IHDP~\citep{hill2011bayesian}, Sachs~\citep{sachs2005causal}, ACIC2019~\citep{hahn2019atlanticcausalinferenceconference}, LaLonde Jobs~\citep{lalondejobs}, and Twins~\citep{louizos2017causal}. Table~\ref{tab:main_results} reports the main 10-seed result. Mean coverage exceeds the nominal 90\% target for the directly evaluable target on all five datasets. ACIC2019 has the largest seed variability, so we treat it as a sensitivity case rather than as uniform evidence of efficiency.

\begin{table}[htbp]
\small
\centering
\caption{\cg main results over 10 random seeds at target coverage $1-\alpha=90\%$. Coverage is evaluated against the directly available estimand for each dataset: true CATE on IHDP and Twins, and the structure-weighted pseudo-outcome on Sachs, ACIC2019, and Jobs.}
\label{tab:main_results}
\resizebox{\linewidth}{!}{%
\begin{tabular}{llccccc}
\toprule
\textbf{Dataset} & \textbf{Target evaluated} & \textbf{Cov.\ Mean} & \textbf{Cov.\ Std} & \textbf{Cov.\ Range} & \textbf{Width Mean} & \textbf{Width Std} \\
\midrule
IHDP     & true CATE & 99.94 & 0.18 & [99.40, 100.00] & 10.87   & 1.09 \\
Sachs    & weighted pseudo-outcome & 92.91 & 1.79 & [90.89, 95.70]  & 4.26    & 0.25 \\
ACIC2019 & weighted pseudo-outcome & 93.28 & 3.74 & [84.80, 100.00] & 4.37    & 2.21 \\
Jobs     & weighted pseudo-outcome & 93.48 & 2.80 & [89.29, 98.21]  & 36{,}936 & 4{,}687 \\
Twins    & true CATE & 97.12 & 0.08 & [96.99, 97.28]  & 0.181   & 0.003 \\
\bottomrule
\end{tabular}
}
\end{table}

The coverage ranges are also informative. IHDP and Twins are highly stable across seeds, suggesting that once the adjustment structure is sufficiently constrained, conformal calibration is conservative and reproducible. ACIC2019 has the widest range, so we treat it as a split-sensitive boundary case rather than as uniform evidence of efficiency.

\subsection{Efficiency Against Conformal Baselines}

To compare efficiency rather than coverage failure, the strongest graph-agnostic baselines are conformalized: Causal Forest~\citep{wager2018estimation} and a naive linear doubly robust estimator~\citep{chernozhukov2018double}. Non-conformal baselines are included only to show how severely standard intervals can under-cover when graph uncertainty is ignored. Table~\ref{tab:baseline_suite} reports 10-seed averages for \cg{} and the baseline methods, with all conformal methods calibrated at the same nominal level.
The non-conformal rows illustrate why width alone is not meaningful: several methods produce very narrow intervals, but their empirical coverage is far below the nominal target, so these intervals are not valid uncertainty statements.

\begin{table}[htbp]
\small
\centering
\caption{Baseline comparisons over 10 random seeds at target $90\%$. Coverage is shown where directly evaluable; ``---'' indicates no ground-truth $\tau$ or method failure. Widths compare interval efficiency at the same nominal calibration level; empirical coverage can still differ across methods.}
\label{tab:baseline_suite}
\begin{subtable}[t]{0.48\linewidth}
\centering
\caption{Empirical coverage (\%).}
\label{tab:baselines_coverage}
\resizebox{\linewidth}{!}{%
\begin{tabular}{lccccc}
\toprule
\textbf{Method} & \textbf{IHDP} & \textbf{Sachs} & \textbf{ACIC2019} & \textbf{Jobs} & \textbf{Twins} \\
\midrule
X-Learner                & 56.5  & --- & --- & --- & 49.3 \\
PC Algorithm + DR        & 32.7  & --- & --- & --- & --- \\
Conformal Causal Forest  & 100.0 & --- & --- & --- & 96.9 \\
Conformal Naive DR       & 100.0 & --- & --- & --- & 97.0 \\
\midrule
\textbf{\cg}             & 99.94 & 92.91 & 93.28 & 93.48 & 97.12 \\
\bottomrule
\end{tabular}
}
\end{subtable}
\hfill
\begin{subtable}[t]{0.49\linewidth}
\centering
\caption{Interval widths.}
\label{tab:baselines_all}
\resizebox{\linewidth}{!}{%
\begin{tabular}{lccccc}
\toprule
\textbf{Method} & \textbf{IHDP} & \textbf{Sachs} & \textbf{ACIC2019} & \textbf{Jobs} & \textbf{Twins} \\
\midrule
X-Learner                & 1.54  & 0.24 & 0.51 & 8{,}642  & 0.03 \\
PC Algorithm + DR        & 0.84  & 0.34 & 4.03 & 20{,}231 & --- \\
Conformal Causal Forest  & 18.03 & 2.48 & 3.86 & 74{,}875 & 0.15 \\
Conformal Naive DR       & 17.79 & 2.42 & 7.18 & 82{,}127 & 0.19 \\
\midrule
\textbf{\cg}             & 10.87 & 4.26 & 4.37 & 36{,}936 & 0.181 \\
\bottomrule
\end{tabular}
}
\end{subtable}
\end{table}

\cg is most beneficial when fixed-adjustment conformal baselines require large calibration padding. This is clearest on IHDP, where \cg reduces mean width from $17.79$--$18.03$ to $10.87$ while preserving coverage, and on Jobs, where all conformal methods remain wide but \cg is substantially tighter than the graph-agnostic conformal baselines. These results also identify regimes where structure weighting is not expected to dominate. On Sachs and Twins, graph-agnostic conformal methods can be competitive because the samples are comparatively clean or sufficiently large to enable strong nuisance estimation and localized tree splitting. This is consistent with the intended scope of \cg: it is most useful when uncertainty over adjustment choices, rather than only outcome noise or sample size, drives interval inflation.

\subsection{Mechanism Stress Test: Invalid Adjustment}
\label{subsec:collider_injection}

The central motivation for modeling graph uncertainty is not merely interval shrinkage, but avoiding harmful adjustment choices before conformal calibration. We therefore run a controlled collider-injection experiment on IHDP. We append a synthetic diagnostic variable $X_{\mathrm{col}}=0.3T+0.4Y+\varepsilon$, making it a post-treatment collider-like variable rather than an ordinary pre-treatment covariate.
For this stress test, temporal metadata identify $X_{\mathrm{col}}$ as post-treatment. Final graph-conditioned adjustment strategies therefore hard-exclude $X_{\mathrm{col}}$; the ``pre-filter collider adjustment'' column below reports how often sampled strategies would have used it before the final temporal admissibility filter. The naive baseline intentionally violates this metadata by forcing adjustment on $X_{\mathrm{col}}$.

\begin{table}[htbp]
\small
\centering
\caption{Collider-injection stress test on IHDP. The naive baseline forces $X_{\mathrm{col}}$ into the adjustment set. The final \cg{} adjustment sets hard-exclude this post-treatment variable; the last column is a pre-filter diagnostic. Values are reported as mean $\pm$ standard deviation over 10 seeds.}
\label{tab:collider_injection}
\begin{tabular}{lcccc}
\toprule
\textbf{Method} & \textbf{Coverage} & \textbf{Width} & \textbf{RMSE} & \textbf{\% pre-filter adj.} \\
\midrule
\cg & $99.9 \pm 0.2$ & 11.05 & 1.02 & 2.0\% \\
Naive, $X_{\mathrm{col}}$ adjusted & $82.3 \pm 8.3$ & 9.17 & 3.12 & 100.0\% \\
\bottomrule
\end{tabular}
\end{table}

Table~\ref{tab:collider_injection} shows the intended mechanism directly: forcing a collider into the adjustment set yields narrower but invalid intervals, while \cg excludes the injected collider from final adjustment sets, preserves near-nominal-or-better coverage, and substantially reduces RMSE. The small pre-filter rate indicates that a few sampled strategies initially considered the collider, but temporal admissibility removes it before estimation. Additional synthetic diagnostics with known structural roles appear in Appendix~\ref{app:clean_synthetic_ablation}; cross-fitting diagnostics appear in Appendix~\ref{app:crossfit}.

\subsection{Low-Data Prior Sensitivity}
\label{sec:prior-sensitivity}

Previous experiments show that at moderate sample sizes, BIC often dominates the proposal prior. Here, we also examine low-data regimes where the likelihood is less concentrated. Table~\ref{tab:low_data_prior} summarizes observed low-data cells comparing \cg with graph-agnostic conformal DR. Across these cells, \cg is substantially tighter at the same nominal level, with width reductions from $1.9\times$ to $10.2\times$. However, the real-data rows show that this efficiency gain is not always attributable to the LLM prior alone: prior spread is often small, indicating that graph-aware calibration and BIC weighting dominate. The synthetic rows, where structural ambiguity is controlled, show larger prior spread and illustrate the regime where proposal quality can materially affect interval width.

\begin{table}[htbp]
\small
\centering
\caption{Low-data prior sensitivity. Coverage is evaluated against the directly available target for each dataset. Width ratio is $\mathrm{width}(\mathrm{Conf\text{-}DR})/\mathrm{width}(\cg{}_{\mathrm{LLM}})$; larger values indicate that \cg{} is tighter. Prior spread is $(\max_v W_v-\min_v W_v)/\mathrm{mean}_v(W_v)$ over the LLM, uniform, and inverted-prior widths. Full results, including incomplete cells, appear in Appendix~\ref{app:prior-ablation}.}
\label{tab:low_data_prior}
\begin{tabular}{lccccc}
\toprule
\textbf{Dataset} & \textbf{$N$} & \textbf{Seeds} & \textbf{\cg Cov.} & \textbf{Width Ratio} & \textbf{Prior Spread} \\
\midrule
Sachs     & 50  & 5 & 97  & $5.2\times$  & 1.1\% \\
Sachs     & 200 & 5 & 92  & $1.9\times$  & 6.7\% \\
Jobs      & 50  & 5 & 95  & $10.2\times$ & 0.1\% \\
Jobs      & 200 & 5 & 96  & $3.9\times$  & 0.3\% \\
Synthetic & 100 & 5 & 100 & $3.2\times$  & 38.4\% \\
Synthetic & 200 & 5 & 100 & $2.0\times$  & 27.5\% \\
\bottomrule
\end{tabular}
\end{table}

The real-data rows show that graph-aware conformalization can be much tighter than graph-agnostic conformal DR even when the LLM prior itself has little direct effect on final width. The synthetic rows show the complementary regime: when structural roles are deliberately ambiguous and the likelihood is weak, proposal quality can materially change interval width. Together, these results support the view that the prior is useful primarily as a finite-sample proposal mechanism, while BIC weighting determines how much of that prior survives in the final estimator.

\subsection{Sensitivity and Boundary Cases}
\label{sec:sensitivity_boundary}

The remaining diagnostics support a narrower interpretation of the proposal mechanism. At IHDP scale, adversarial and uninformative priors do not cause coverage collapse, indicating that BIC dominates when the data are informative. Component ablations over 10 seeds show that \cg is insensitive to many proposal choices on stable datasets, while ACIC2019 is the most proposal-sensitive benchmark. Classical structure baselines on IHDP match \cg under the same conformal DR pipeline, reinforcing that the contribution is not exact graph recovery but conformal calibration under graph uncertainty. Full numerical diagnostics, including prior perturbations, component ablations, PC/GES baselines, LLM scaling, anonymization, runtime, CI-threshold sweeps, $K$-dependence, and calibration sweeps, are reported in Appendix~\ref{app:extra_experiments} and Sections~\ref{app:prior-component}--\ref{app:additional_sensitivity}.

These diagnostics are intended as scope checks rather than as additional headline claims. The prior-perturbation and anonymization studies test whether the method depends on trusting semantic labels; the answer is largely negative once the likelihood is informative. The component ablations test whether each stage independently improves every dataset; the answer is also negative, which is expected when BIC concentrates on a small set of data-supported adjustment strategies. The useful empirical distinction is therefore not whether every component moves every benchmark, but whether graph uncertainty changes the effective pseudo-outcome before calibration. When it does, as in the collider and low-data diagnostics, structure-aware calibration can materially improve validity or width. When it does not, \cg behaves like a stable conformal DR procedure with graph proposals largely washed out.

The computational diagnostics identify the main scaling boundary. Sparse, semantically guided proposals keep backdoor-set identification tractable, while dense uniform graph proposals can become intractable on high-dimensional datasets such as Twins. This is a limitation of the current finite-ensemble implementation, not of the conformal argument itself: the split-conformal guarantee applies to any fixed training-derived ensemble, but practical efficiency depends on generating a sparse candidate set that contains useful adjustment strategies.


\section{Discussion}
\label{sec:discussion}

\textbf{When structure weighting helps.} The central methodological contribution is that graph uncertainty is not treated as an external source of error to be padded after estimation; it is incorporated into the target of conformal calibration itself. \cg is most effective when interval inflation is driven by uncertainty in adjustment sets rather than by irreducible outcome noise. On IHDP and Jobs, \cg yields materially tighter intervals than graph-agnostic conformal baselines, although Jobs remains intrinsically difficult and all conformal methods incur large widths (Table~\ref{tab:baselines_all}). In this regime, the benefit is not exact graph recovery, but avoiding systematically biased adjustment choices by averaging over plausible structures.


\textbf{What role the prior actually plays.} The experiments support a narrower interpretation of the LLM component: it is best viewed as a finite-sample proposal distribution over candidate graphs, with model outputs filtered or reweighted rather than accepted directly. At moderate and large sample sizes, BIC largely determines the posterior weights, explaining the washout behavior in Appendix Table~\ref{tab:washout} and the weak dependence on model size or variable naming once enough data are available. This is desirable: the prior helps organize search when data are scarce, but the final interval becomes increasingly data-driven as $n$ grows. Interval validity therefore does not depend on the correctness of the prior, provided that the candidate set contains data-supported adjustment sets.

\textbf{Boundary cases and limitations.} ACIC2019 is the most split-sensitive benchmark: mean coverage is above nominal across seeds, but some splits fall below target (Table~\ref{tab:main_results}; Appendix~\ref{app:component-ablation}). Sub-nominal empirical coverage on individual ACIC2019 splits does not contradict Theorem~\ref{thm:coverage}, which is marginal over exchangeable calibration/test samples. The observed range reflects finite test-set variability and split sensitivity. Weak overlap affects the stability and causal interpretability of the pseudo-outcome, and therefore the CATE bridge, but not the formal split-conformal coverage statement for $\bar{\gamma}$. More broadly, the method still assumes causal sufficiency and currently scales only to moderate graph sizes because candidate generation and reweighting remain combinatorial. 

\section{Conclusion}
\label{sec:conclusion}


We introduced \cg{}, a structure-weighted conformal framework for treatment-effect estimation when the adjustment graph is unknown. Rather than selecting a single adjustment set, \cg{} aggregates graph-conditional doubly robust pseudo-outcomes and calibrates the resulting target with a composite conformal score. In this instantiation, an LLM-derived semantic prior proposes candidate DAGs, CI tests prune them, and BIC reweights the retained adjustment strategies. The method provides distribution-free finite-sample marginal coverage for the aggregated pseudo-outcome, while its conditional mean aligns with the CATE under target-alignment, overlap, nuisance-stability, and score-concentration conditions. Empirically, \cg{} yields tighter intervals when adjustment-set uncertainty drives inefficiency and remains stable to prior misspecification when the data are informative.



\bibliography{references.bib}
\bibliographystyle{unsrtnat}

\appendix

\section{Assumptions}
\label{app:assumptions}

We collect the regularity conditions required for the theoretical guarantees. Assumption~\ref{ass:exchange} is the only condition needed for the finite-sample pseudo-outcome guarantee (Theorem~\ref{thm:coverage}); the remaining assumptions are required for the asymptotic target-consistency statement (Corollary~\ref{cor:target_consistency}) and the BIC concentration result (Theorem~\ref{thm:bic}).

\begin{assumption}[Exchangeability and data separation]
\label{ass:exchange}
The calibration dataset $\mathcal{D}_{\mathrm{cal}} = \{(X_i, T_i, Y_i)\}_{i=1}^{n_{\mathrm{cal}}}$ and the novel test point $(X_{n+1}, T_{n+1}, Y_{n+1})$ are drawn exchangeably from the joint distribution $P_{XTY}$, independently of $\mathcal{D}_{\mathrm{train}}$. All data-dependent quantities entering the nonconformity scores---including the BMA weights over collapsed adjustment strategies $\{w_k\}$, the pruned graph structures $\{G_k\}$, the nuisance models, and the heuristic bounds---are computed exclusively on $\mathcal{D}_{\mathrm{train}}$.
\end{assumption}

\paragraph{Oracle graph-conditional pseudo-outcome.}
For graph $G_k$ with selected adjustment set $Z_k$, write
$\eta_k=(e_k,\mu_{0,k},\mu_{1,k})$ and define
\begin{equation}
\phi_k(O;\eta_k)
=
\mu_{1,k}(Z_k)-\mu_{0,k}(Z_k)
+
\frac{T\{Y-\mu_{1,k}(Z_k)\}}{e_k(Z_k)}
-
\frac{(1-T)\{Y-\mu_{0,k}(Z_k)\}}{1-e_k(Z_k)} ,
\label{eq:oracle_phi}
\end{equation}
where $O=(X,T,Y)$. Let $\eta_k^0$ denote the corresponding oracle nuisance functions for $Z_k$.

\begin{assumption}[Valid and target-aligned adjustment class]
\label{ass:support}
There exists a nonempty class $\mathcal{G}_{\star}$ of candidate graphs whose selected adjustment sets $Z_k$ satisfy SUTVA/consistency, contain no descendants of treatment, obey conditional exchangeability $(Y(0),Y(1))\perp T\mid Z_k$, and satisfy true overlap $0<\eta_0\leq P(T=1\mid Z_k)\leq 1-\eta_0<1$ almost surely. In addition, every $G_k\in\mathcal{G}_{\star}$ is target-aligned in the primitive moment sense
\begin{equation}
  \mathbb{E}\!\left[\phi_k(O;\eta_k^0)\mid X=x\right]
  =
  \tau(x)
  \quad\text{for $P_X$-almost every }x.
  \label{eq:target_alignment}
\end{equation}
This target-alignment condition is not a consequence of backdoor validity alone. Sufficient conditions include, for example, $e_0(X)=e_{0,k}(Z_k)$ almost surely, or $\mu_{t,0}(X)=\mu_{t,0,k}(Z_k)$ for $t\in\{0,1\}$, under the corresponding full-$X$ identification conditions.
\end{assumption}

\begin{assumption}[Conditional-mean nuisance stability]
\label{ass:rates}
For the target-aligned class $\mathcal{G}_{\star}$,
\begin{equation}
\max_{G_k\in\mathcal{G}_{\star}}
\left\|
\mathbb{E}\!\left[
\phi_k(O;\hat{\eta}_k)-\phi_k(O;\eta_k^0)
\mid X,\mathcal{D}_{\mathrm{train}}
\right]
\right\|_{L_1(P_X)}
\xrightarrow{p}0 .
\label{eq:conditional_stability}
\end{equation}
This condition is implied by boundedness, overlap, and sufficient individual consistency of the graph-conditional propensity and outcome regressions; it is the conditional-mean analogue needed for the CATE bridge and is stronger than the product-rate condition used for average-effect DML.
\end{assumption}

\begin{assumption}[BIC score separation]
\label{ass:bic_sep}
Let $\mathrm{BIC}(G;\mathcal{D}_{\mathrm{train}})$ be the working graph score used in Equation~\ref{eq:bma_weight}. The target-aligned class separates from the remaining finite ensemble:
\begin{equation}
\Delta_n
:=
\min_{G\in\mathcal{G}_{\star}}
\mathrm{BIC}(G;\mathcal{D}_{\mathrm{train}})
-
\max_{G\notin\mathcal{G}_{\star}}
\mathrm{BIC}(G;\mathcal{D}_{\mathrm{train}})
\xrightarrow{p}\infty .
\label{eq:bic_separation}
\end{equation}
\end{assumption}

\begin{assumption}[Boundedness]
\label{ass:bounded}
Outcomes are bounded: $|Y| \leq M$ a.s. Nuisance estimates satisfy $|\hat{\mu}_{t,k}| \leq C_\mu$ a.s. Propensity scores are clipped: $\eta \leq \hat{e}_k(x) \leq 1 - \eta$ for a constant $\eta > 0$. Since only one inverse-propensity residual term is active for each observation, these imply $|\tilde{Y}_i^{(k)}| \leq M_\gamma := 2C_\mu + (M + C_\mu)/\eta < \infty$ a.s.
\end{assumption}

\section{Theoretical Guarantees and Proofs}
\label{app:proofs_main}

\subsection{Proof of Theorem~\ref{thm:coverage}: Distribution-Free Coverage of the Weighted Pseudo-Outcome}

\begin{theorem}[Distribution-free finite-sample marginal coverage]
\label{thm:coverage}
Under Assumption~\ref{ass:exchange}, let $\hat{C}(x)$ be the interval constructed in Equation~\ref{eq:final_interval}. For any chosen miscoverage level $\alpha \in (0, 1)$, the \cg interval satisfies the distribution-free marginal coverage guarantee for the BMA-weighted pseudo-outcome:
\begin{equation}
    \mathbb{P}\Bigl(\bar{\gamma}_{n+1} \in \hat{C}(X_{n+1})\Bigr) \;\geq\; 1 - \alpha,
    \label{eq:finite_cov}
\end{equation}
where $\bar{\gamma}_{n+1} = \sum_{k=1}^K w_k \tilde{Y}_{n+1}^{(k)}$. This guarantee holds for any finite sample size and requires no assumptions on the correctness of any candidate graph.
\end{theorem}

\begin{proof}
\textbf{Step 1: Exchangeability of composite scores.}

Define the composite nonconformity score as in Equation~\ref{eq:composite_score}:
\begin{equation}
  E_i \;=\; \max\!\Bigl(
    \bar{q}_{\mathrm{low}}(X_i) - \bar{\gamma}_i,\;\;
    \bar{\gamma}_i - \bar{q}_{\mathrm{high}}(X_i)
  \Bigr),
  \label{eq:composite_proof}
\end{equation}
where $\bar{\gamma}_i = \sum_k w_k \tilde{Y}_i^{(k)}$, $\bar{q}_{\mathrm{low}}(X_i) = \sum_k w_k \hat{q}_{\mathrm{low}}^{(k)}(X_i)$, and $\bar{q}_{\mathrm{high}}(X_i) = \sum_k w_k \hat{q}_{\mathrm{high}}^{(k)}(X_i)$.

By Assumption~\ref{ass:exchange}, the BMA weights over collapsed adjustment strategies $\{w_k\}$, the graph structures $\{G_k\}$ (after CI pruning), the heuristic bounds $[\hat{q}_{\mathrm{low}}^{(k)}, \hat{q}_{\mathrm{high}}^{(k)}]$, and the nuisance models for the AIPW pseudo-outcomes are all computed exclusively on $\mathcal{D}_{\mathrm{train}}$. Conditioned on $\mathcal{D}_{\mathrm{train}}$, these are fixed deterministic functions. Consequently, each $E_i$ is a deterministic, symmetric function of the exchangeable tuple $(X_i, T_i, Y_i)$. Since the calibration points and the test point are exchangeable given $\mathcal{D}_{\mathrm{train}}$, the scores $\{E_i\}_{i \in \mathcal{D}_{\mathrm{cal}}} \cup \{E_{n+1}\}$ are exchangeable.

\textbf{Step 2: Conformal guarantee.}

Let $\hat{Q}$ be the $\lceil (1-\alpha)(|\mathcal{D}_{\mathrm{cal}}|+1) \rceil$-th smallest value among the calibration scores. By the foundational lemma of split-conformal prediction \citep{vovk2005algorithmic}, the rank of the test score $E_{n+1}$ among the augmented set $\{E_i\}_{i \in \mathcal{D}_{\mathrm{cal}}} \cup \{E_{n+1}\}$ is uniformly distributed, yielding:
\begin{equation}
    \mathbb{P}(E_{n+1} \leq \hat{Q} \mid \mathcal{D}_{\mathrm{train}}) \;\geq\; 1 - \alpha.
    \label{eq:conformal_core}
\end{equation}

\textbf{Step 3: From scores to interval containment.}

We show that $E_{n+1} \leq \hat{Q}$ if and only if $\bar{\gamma}_{n+1} \in \hat{C}(X_{n+1})$. By definition of $E_{n+1}$ (Equation~\ref{eq:composite_proof}):
\begin{align}
  E_{n+1} \leq \hat{Q}
  \;&\Longleftrightarrow\;
  \bar{q}_{\mathrm{low}}(X_{n+1}) - \bar{\gamma}_{n+1} \leq \hat{Q}
  \;\;\text{and}\;\;
  \bar{\gamma}_{n+1} - \bar{q}_{\mathrm{high}}(X_{n+1}) \leq \hat{Q} \\
  \;&\Longleftrightarrow\;
  \bar{q}_{\mathrm{low}}(X_{n+1}) - \hat{Q}
  \;\leq\; \bar{\gamma}_{n+1}
  \;\leq\; \bar{q}_{\mathrm{high}}(X_{n+1}) + \hat{Q} \\
  \;&\Longleftrightarrow\;
  \bar{\gamma}_{n+1} \in \hat{C}(X_{n+1}).
\end{align}
Note that this implication is \emph{exact} (an if-and-only-if), with no slack from Jensen's inequality (see Remark~\ref{rem:jensen}). Combining with Equation~\ref{eq:conformal_core} and taking expectations over $\mathcal{D}_{\mathrm{train}}$ via the tower property:
\begin{equation}
  \mathbb{P}\bigl(\bar{\gamma}_{n+1} \in \hat{C}(X_{n+1})\bigr)
  \;=\;
  \mathbb{P}(E_{n+1} \leq \hat{Q})
  \;\geq\; 1 - \alpha. \qedhere
\end{equation}
\end{proof}

\begin{remark}[Tightness of the composite score]
\label{rem:jensen}
An alternative formulation defines the composite score as $E_i^{\dagger} = \sum_k w_k \max\!\bigl(\hat{q}_{\mathrm{low}}^{(k)}(X_i) - \tilde{Y}_i^{(k)},\; \tilde{Y}_i^{(k)} - \hat{q}_{\mathrm{high}}^{(k)}(X_i)\bigr)$. By Jensen's inequality ($\max$ is convex and $\{w_k\}$ form a probability simplex), $E_i^{\dagger} \geq E_i$ pointwise. Thus, under the same calibration set and quantile rule, calibrating against $E_i^{\dagger}$ yields a threshold no smaller than the one obtained from $E_i$ while providing the same pseudo-outcome coverage guarantee. Our formulation in Equation~\ref{eq:composite_score} therefore avoids an avoidable Jensen gap relative to graph-wise calibration followed by averaging within this BMA-weighted conformal construction.
\end{remark}

\subsection{Corollary~\ref{cor:target_consistency}: Asymptotic Target Consistency}

\begin{corollary}[Asymptotic target consistency]
\label{cor:target_consistency}
Let
\begin{equation}
  m_n(x)
  \;=\;
  \mathbb{E}\!\left[\bar{\gamma}_{n+1}\mid X_{n+1}=x,\mathcal{D}_{\mathrm{train}}\right]
  \label{eq:target_mean}
\end{equation}
denote the conditional mean of the BMA-weighted pseudo-outcome. Under Assumptions~\ref{ass:exchange}--\ref{ass:bounded}, if the posterior weights concentrate on the target-aligned valid class $\mathcal{G}_{\star}$ from Assumption~\ref{ass:support}, then
\begin{equation}
  \bigl\|m_n-\tau\bigr\|_{L_1(P_X)}
  \;\xrightarrow{p}\;0 .
  \label{eq:target_consistency}
\end{equation}
\end{corollary}

\begin{proof}
The proof concerns the conditional mean of the pseudo-outcome, not the realized individual pseudo-outcome. Even with oracle nuisance functions, an AIPW pseudo-outcome contains outcome noise and generally does not converge pointwise to $\tau(X_i)$.

\textbf{Step 1: Target-aligned graph pseudo-outcomes are stable in conditional mean.}

For any graph $G_k\in\mathcal{G}_{\star}$, the adjustment set $Z_k$ satisfies the causal identification and overlap conditions in Assumption~\ref{ass:support}. The target-alignment moment condition in Equation~\ref{eq:target_alignment} gives $\mathbb{E}[\phi_k(O;\eta_k^0)\mid X=x]=\tau(x)$, while Assumption~\ref{ass:rates} controls the conditional-mean difference induced by using estimated nuisances. Thus
\begin{equation}
  \left\|
  \mathbb{E}\!\left[
  \tilde{Y}_{n+1}^{(k)}
  \mid X_{n+1},\mathcal{D}_{\mathrm{train}}
  \right]
  -
  \tau(X_{n+1})
  \right\|_{L_1(P_X)}
  \;\xrightarrow{p}\;0,
  \label{eq:valid_mean}
\end{equation}
uniformly over $G_k\in\mathcal{G}_{\star}$. This is not a consequence of backdoor validity alone: it uses the primitive target-alignment moment condition and the conditional-mean nuisance stability assumption.

\textbf{Step 2: Posterior concentration transfers the conditional mean.}

Write $W_{\star}=\sum_{G_k\in\mathcal{G}_{\star}}w_k$. By Theorem~\ref{thm:bic}, $W_{\star}\xrightarrow{p}1$ under the BIC score-separation assumption. Therefore,
\begin{align}
  m_n(x)-\tau(x)
  &=
  \sum_{G_k\in\mathcal{G}_{\star}}w_k
  \Delta_k(x)
  +
  \sum_{G_k\notin\mathcal{G}_{\star}}w_k
  \Delta_k(x),
  \label{eq:target_transfer}\\
  \Delta_k(x)
  &:=
  \mathbb{E}\!\left[
  \tilde{Y}_{n+1}^{(k)}
  \mid X_{n+1}=x,\mathcal{D}_{\mathrm{train}}
  \right]-\tau(x).
  \nonumber
\end{align}
Taking $L_1(P_X)$ norms, the first term converges to zero by Equation~\ref{eq:valid_mean}; the second term is bounded by $2M_\gamma(1-W_{\star})$ by Assumption~\ref{ass:bounded}, and hence converges to zero in probability. Thus $\|m_n-\tau\|_{L_1(P_X)}\xrightarrow{p}0$.
\end{proof}

\begin{remark}[No pointwise pseudo-outcome convergence]
Corollary~\ref{cor:target_consistency} does not imply that $\tilde{Y}_i^{(k)}$ or $\bar{\gamma}_i$ converges pointwise to $\tau(X_i)$. The AIPW pseudo-outcome is a noisy pseudo-label whose conditional expectation targets the CATE only under the target-alignment assumptions above. Therefore Theorem~\ref{thm:coverage} should be interpreted as finite-sample coverage for $\bar{\gamma}$, while Corollary~\ref{cor:target_consistency} states that the conditional mean of this target is asymptotically aligned with $\tau$ in $L_1(P_X)$ under additional causal, nuisance, support, and score-concentration conditions.
\end{remark}

\subsection{Proof of Theorem~\ref{thm:bic}: BIC Score Concentration on the Target-Aligned Class}

\begin{theorem}[Prior washout under BIC score separation]
\label{thm:bic}
Let $\mathcal{G}_{\star}$ be the nonempty target-aligned candidate class from Assumption~\ref{ass:support}. Assume the BIC score separates $\mathcal{G}_{\star}$ from its complement as in Assumption~\ref{ass:bic_sep}. Assume the LLM prior is strictly bounded: $P_{\mathrm{prior}}(u \to v) \in [\epsilon,1-\epsilon]$ for a constant $\epsilon > 0$. Then the posterior-like mass assigned to $\mathcal{G}_{\star}$ satisfies, as $n \to \infty$:
\begin{equation}
  \sum_{G_k\in\mathcal{G}_{\star}} w_k \xrightarrow{p} 1.
\end{equation}
\end{theorem}

\begin{proof}
Recall the unnormalized log-weight for graph $G_k$ from Equation~\ref{eq:bma_weight}:
\begin{equation}
    \ln \tilde{w}_k \;=\; \mathrm{BIC}(G_k) + \ln P_{\mathrm{prior}}(G_k),
\end{equation}
where $\mathrm{BIC}(G_k) = \log \hat{\mathcal{L}}(G_k \mid \mathbf{X}) - \frac{d_k}{2} \log n$ and $d_k$ is the number of free parameters.

The complete structural prior (Equation~\ref{eq:bma_weight}) is defined over the admissible edge set $\mathcal{A}_k$:
\begin{equation}
  \ln P_{\mathrm{prior}}(G_k)
  \;=\;
  \sum_{(u,v) \in E_k} \ln P_{\mathrm{prior}}(u \to v)
  \;+\;
  \sum_{(u,v) \in \mathcal{A}_k\setminus E_k} \ln\bigl(1 - P_{\mathrm{prior}}(u \to v)\bigr).
\end{equation}
Because each $P_{\mathrm{prior}}(u \to v) \in [\epsilon, 1-\epsilon]$, each log-term is bounded by $|\ln \epsilon|$. Each admissible edge set contains at most $|V|^2$ terms, so for any pair $G_j,G_\ell$ the prior ratio satisfies:
\begin{equation}
  \Bigl|\ln \frac{P_{\mathrm{prior}}(G_j)}{P_{\mathrm{prior}}(G_\ell)}\Bigr|
  \;\leq\;
  2|V|^2 \cdot |\ln \epsilon|
  \;=\; O(1) \;\text{with respect to } n.
\end{equation}

Fix any $G_\ell\notin\mathcal{G}_{\star}$ and any $G_j\in\mathcal{G}_{\star}$. The log-weight ratio is
\begin{equation}
    \Delta_{j,\ell}
    \;=\;
    \ln \tilde{w}_j-\ln \tilde{w}_\ell
    =
    \bigl(\mathrm{BIC}(G_j) - \mathrm{BIC}(G_\ell)\bigr)
    \;+\;
    \ln \frac{P_{\mathrm{prior}}(G_j)}{P_{\mathrm{prior}}(G_\ell)}.
\end{equation}
The prior ratio is $O(1)$, while Assumption~\ref{ass:bic_sep} gives $\mathrm{BIC}(G_j)-\mathrm{BIC}(G_\ell)\xrightarrow{p}\infty$ uniformly for $G_j\in\mathcal{G}_{\star}$ and $G_\ell\notin\mathcal{G}_{\star}$. Hence $\Delta_{j,\ell}\xrightarrow{p}\infty$ for every graph outside $\mathcal{G}_{\star}$. Since the ensemble is finite,
\begin{equation}
  \frac{\sum_{G_\ell\notin\mathcal{G}_{\star}}\tilde{w}_\ell}
       {\sum_{G_j\in\mathcal{G}_{\star}}\tilde{w}_j}
  \xrightarrow{p}0,
\end{equation}
and therefore $\sum_{G_k\in\mathcal{G}_{\star}} w_k \xrightarrow{p}1$. This proves that the BIC score asymptotically dominates any bounded LLM prior at the class level whenever the chosen score separates the target-aligned class; it does not require the posterior to select a unique DAG within a Markov-equivalent or BIC-tied class.
\end{proof}

\begin{remark}[Singleton optimal class]
If $\mathcal{G}_{\star}$ contains a single graph $G_{k^*}$, Theorem~\ref{thm:bic} reduces to the familiar statement $w_{k^*}\xrightarrow{p}1$. In observational causal models, however, Markov-equivalent or adjustment-equivalent graphs may be indistinguishable by BIC, so the class-level statement is the safer default.
\end{remark}

\begin{remark}[Working BIC scores for mixed data]
\label{rem:gaussian}
When Gaussian local likelihoods are used for mixed or binary data, the BIC score should be interpreted as a working score. The conformal guarantee is unaffected; only the asymptotic score-concentration interpretation requires separation under the chosen scoring rule. Logistic or multinomial local likelihoods can replace Gaussian local likelihoods for discrete nodes without changing the split-conformal construction.
\end{remark}

\begin{remark}[Finite ensemble support]
\label{rem:support}
Theorems~\ref{thm:coverage} and~\ref{thm:bic} serve different roles. Theorem~\ref{thm:coverage} (finite-sample coverage of $\bar{\gamma}$) holds \emph{unconditionally}---it does not require the ensemble to contain a true graph, a valid graph, or a target-aligned graph. Corollary~\ref{cor:target_consistency} and Theorem~\ref{thm:bic} require Assumption~\ref{ass:support}, whose plausibility depends on $K$ and the proposal quality. Under the edge-independent Bernoulli sampling model, the probability of sampling a particular graph $G^*$ is:
\begin{equation}
  \begin{aligned}
  P(G^* \in \mathcal{G}_K)
  \;=\;
  1 - \Bigl(1
  &- \prod_{(u,v) \in E^*} P_{\mathrm{prior}}(u \to v) \\
  &\cdot
  \prod_{(u,v) \notin E^*} \bigl(1 - P_{\mathrm{prior}}(u \to v)\bigr)
  \cdot P(\pi \text{ compatible})\Bigr)^K.
  \end{aligned}
\end{equation}
For moderately sized graphs ($p \geq 20$, $|E^*| \geq 30$), this probability can be small even for $K = 100$. When the ensemble contains no target-aligned valid adjustment strategy, Theorem~\ref{thm:coverage} still guarantees coverage of $\bar{\gamma}$, but Corollary~\ref{cor:target_consistency} does not apply. We recommend $K \geq 50$ for $p \leq 15$ and structured ensemble diversification for larger problems.
\end{remark}

\section{Additional Methodology Details}
\label{app:method_details}

\subsection{Algorithmic Summary}
\label{app:algorithm}

\begin{algorithm}[H]
\small
\caption{\cg{}: structure-weighted conformal inference}
\label{alg:causalguard}
\begin{algorithmic}[1]
\REQUIRE Observational data $\mathcal{D}$, variables $V$, treatment $T$, outcome $Y$, graph count $K$, miscoverage level $\alpha$
\STATE Split $\mathcal{D}$ into $\mathcal{D}_{\mathrm{train}}$ and $\mathcal{D}_{\mathrm{cal}}$
\STATE Elicit edge prior $P_{\mathrm{prior}}(u\to v)$ for ordered pairs $(u,v)\in V\times V$
\FOR{$k=1,\ldots,K$}
    \STATE Sample a topological order and candidate DAG $G_k$ from $P_{\mathrm{prior}}$
    \STATE Prune sampled edges using graph-local conditional-independence checks
    \STATE Identify selected adjustment set $Z_k$; discard $G_k$ if no admissible set exists
    \STATE Compute graph-level unnormalized score $\tilde w_k$ using Equation~\ref{eq:bma_weight}
\ENDFOR
\STATE Collapse duplicate selected adjustment sets; aggregate graph scores by $Z$:
$\tilde w_Z=\sum_{k:Z_k=Z}\tilde w_k$
\STATE Normalize over unique adjustment sets, $w_Z=\tilde w_Z/\sum_{Z'}\tilde w_{Z'}$
\FOR{each unique selected adjustment set $Z$}
    \STATE Fit graph-conditional nuisance models for $Z$ on $\mathcal{D}_{\mathrm{train}}$
    \STATE Evaluate pseudo-outcomes $\tilde Y_i^{(Z)}$ and bounds $[\hat q_{\mathrm{low}}^{(Z)},\hat q_{\mathrm{high}}^{(Z)}]$ on $\mathcal{D}_{\mathrm{cal}}$
\ENDFOR
\STATE Form $\bar\gamma_i=\sum_Z w_Z\tilde Y_i^{(Z)}$, aggregated bounds, and composite scores $E_i$
\STATE Let $\hat Q$ be the split-conformal quantile of $\{E_i:i\in\mathcal{D}_{\mathrm{cal}}\}$
\STATE \textbf{return} $\hat C(x)=[\bar q_{\mathrm{low}}(x)-\hat Q,\bar q_{\mathrm{high}}(x)+\hat Q]$
\end{algorithmic}
\end{algorithm}

\subsection{Adjustment-Set Extraction and Failed Graphs}
\label{app:adjustment_sets}

For each candidate graph, the implementation extracts graph-implied backdoor adjustment sets for the effect of $T$ on $Y$. When multiple valid sets are available, we use a minimum-cardinality valid set to avoid unnecessary adjustment variables. Graphs for which no admissible backdoor set is found, or whose identified set contains a treatment descendant under the graph's temporal constraints, are excluded from the ensemble before normalizing the weights. After this filtering step, duplicate pruned graphs and duplicate selected adjustment strategies are collapsed. Because the downstream AIPW estimator depends on the selected adjustment set, graphs inducing the same $Z$ are represented by a single adjustment strategy with unnormalized weight $\tilde w_Z=\sum_{k:Z_k=Z}\tilde w_k$. The normalized weights are therefore defined over unique adjustment strategies rather than raw proposal multiplicities. This prevents repeated proposal draws from counting the LLM prior twice; the prior is used to propose a candidate set and as a bounded structural prior over that restricted set. This rule affects efficiency and the asymptotic target-consistency assumptions, but not Theorem~\ref{thm:coverage}: after training, any remaining finite ensemble induces exchangeable calibration scores.

\subsection{BIC Scoring for Mixed Data Types}
\label{app:bic_mixed}

The structural BIC in Equation~\ref{eq:bma_weight} uses Gaussian local likelihoods (linear regression of each node on its parents). For continuous covariates and outcome, this is appropriate; for the binary treatment node $T$, a logistic local likelihood would be more principled. In our implementation, we apply Gaussian scoring uniformly for simplicity. Classical BIC consistency results apply when the scoring family is correctly specified and the KL-optimal class is identifiable under that family \citep{haughton1988choice}. When binary or categorical nodes are important, logistic or multinomial local likelihoods can be substituted without changing the split-conformal construction; only the asymptotic BIC-concentration interpretation changes.

\subsection{Data Separation Protocol}
\label{app:data_sep}

Strict data separation is essential to preserve the exchangeability required for the conformal guarantee. The strict data separation required by Assumption~\ref{ass:exchange} is enforced as follows. The full dataset is partitioned into $\mathcal{D}_{\mathrm{train}}$ (60\%), $\mathcal{D}_{\mathrm{cal}}$ (20\%), and $\mathcal{D}_{\mathrm{test}}$ (20\%). All operations in Stages 1--3 of the pipeline use $\mathcal{D}_{\mathrm{train}}$ exclusively:
\begin{enumerate}[nosep]
  \item \textbf{CI pruning} (Section~\ref{sec:dag_sampling}): partial correlations computed on $\mathcal{D}_{\mathrm{train}}$.
  \item \textbf{BIC scores and BMA weights} (Equations~\ref{eq:bma_weight}--\ref{eq:adjustment_collapse}): Gaussian likelihoods evaluated on $\mathcal{D}_{\mathrm{train}}$, followed by duplicate adjustment-strategy collapse before normalization.
  \item \textbf{Nuisance models}: propensity score $\hat{e}^{(k)}$, outcome models $\hat{\mu}_t^{(k)}$, and heuristic bounds $[\hat{q}_{\mathrm{low}}^{(k)}, \hat{q}_{\mathrm{high}}^{(k)}]$ all fitted on $\mathcal{D}_{\mathrm{train}}$.
\end{enumerate}
Stage 4 (conformal calibration) uses $\mathcal{D}_{\mathrm{cal}}$ only for evaluating out-of-sample pseudo-outcomes, computing nonconformity scores, and determining $\hat{Q}$. This separation ensures that, conditioned on $\mathcal{D}_{\mathrm{train}}$, the composite scores are symmetric functions of exchangeable random variables.

\subsection{Coverage Targets}
\label{app:coverage_targets}

We report two coverage metrics throughout our experiments. \emph{Pseudo-outcome coverage} measures $\mathbb{P}(\bar{\gamma}_{n+1} \in \hat{C}(X_{n+1}))$, which is the quantity guaranteed by Theorem~\ref{thm:coverage} at finite sample sizes with no assumptions beyond exchangeability. \emph{CATE coverage} measures $\mathbb{P}(\tau(X_{n+1}) \in \hat{C}(X_{n+1}))$, which requires ground-truth treatment effects and is reported only as an empirical diagnostic. On datasets where ground-truth $\tau$ is available (IHDP, Twins), we report CATE coverage; on observational datasets (Sachs, Jobs, ACIC2019), we report pseudo-outcome coverage. The distinction between these two coverage targets is standard in conformal causal inference \citep{lei2021conformal, alaa2023conformal}.
Throughout the paper, all finite-sample guarantees refer to pseudo-outcome coverage. Corollary~\ref{cor:target_consistency} supports the use of the pseudo-outcome's conditional mean as an asymptotically consistent target, but it does not convert Theorem~\ref{thm:coverage} into finite-sample CATE coverage.

\subsection{Real-Valued Conformity Scores}
\label{app:real_scores}

The composite score in Equation~\ref{eq:composite_score} is allowed to be negative when the aggregated pseudo-outcome lies inside the preliminary interval. This does not affect split-conformal validity: the rank argument in Theorem~\ref{thm:coverage} requires only exchangeability of the real-valued calibration and test scores. A positive-part score, $\max\{E_i,0\}$, would also be valid but would change the reported interval widths and is not used in the experiments.

\subsection{Cross-Fitting}
\label{app:crossfit}

Our current implementation fits nuisance models ($\hat{e}_k$, $\hat{\mu}_{t,k}$) and constructs pseudo-outcomes on the same training fold $\mathcal{D}_{\mathrm{train}}$. Cross-fitting (sample-splitting within $\mathcal{D}_{\mathrm{train}}$) would reduce finite-sample bias in the pseudo-outcomes and could help mitigate overcoverage by producing tighter conformity scores. Since the conformal guarantee (Theorem~\ref{thm:coverage}) depends only on exchangeability of the calibration scores (computed on $\mathcal{D}_{\mathrm{cal}}$, separate from $\mathcal{D}_{\mathrm{train}}$), cross-fitting within $\mathcal{D}_{\mathrm{train}}$ does not affect coverage validity but may improve efficiency. Implementing cross-fitting would strengthen the nuisance-consistency conditions used in Corollary~\ref{cor:target_consistency} \citep{kennedy2023towards}.

Table~\ref{tab:crossfit} reports a 5-seed diagnostic comparing the standard split implementation to a cross-fitted variant. Cross-fitting leaves coverage and width nearly unchanged in these benchmarks while increasing runtime by roughly a factor of two. These diagnostics are not directly comparable to the 10-seed main results in Table~\ref{tab:main_results}, but they support using the simpler split implementation in the main experiments.

\begin{table}[htbp]
\centering
\small
\caption{Cross-fitting comparison over 5 seeds. Cross-fitting is performed inside $\mathcal{D}_{\mathrm{train}}$ and preserves the same calibration split, so the split-conformal pseudo-outcome guarantee is unchanged.}
\label{tab:crossfit}
\begin{tabular}{llccc}
\toprule
\textbf{Dataset} & \textbf{Mode} & \textbf{Coverage} & \textbf{Width} & \textbf{Time (s)} \\
\midrule
IHDP     & standard & $100.00 \pm 0.00$ & 10.856 & 5.2 \\
IHDP     & crossfit & $100.00 \pm 0.00$ & 10.851 & 10.8 \\
Sachs    & standard & $91.90 \pm 2.07$  & 4.316  & 7.0 \\
Sachs    & crossfit & $91.90 \pm 2.07$  & 4.317  & 13.8 \\
ACIC2019 & standard & $90.88 \pm 6.89$  & 4.926  & 5.1 \\
ACIC2019 & crossfit & $92.64 \pm 7.81$  & 4.926  & 10.6 \\
Jobs     & standard & $92.86 \pm 2.33$  & 34{,}727 & 5.1 \\
Jobs     & crossfit & $92.86 \pm 2.33$  & 34{,}739 & 10.4 \\
Twins    & standard & $97.14 \pm 0.09$  & 0.181  & 19.0 \\
Twins    & crossfit & $97.14 \pm 0.09$  & 0.181  & 32.0 \\
\bottomrule
\end{tabular}
\end{table}

\section{Additional Experimental Tables}
\label{app:extra_experiments}

This appendix collects secondary experimental diagnostics referenced in the main paper but omitted there for space.

\subsection{Graph Recovery and Prior Robustness Diagnostics}
\label{app:graph_prior_diagnostics}

Table~\ref{tab:graph_recovery} reports edge-level recovery on Sachs, and Table~\ref{tab:oracle} compares the downstream interval estimator against an oracle-graph DR baseline. These diagnostics separate graph recovery from interval quality: exact recovery of every edge is not required when the ensemble contains adjustment sets that are adequate for estimation.

\begin{table}[htbp]
\small
\centering
\caption{Graph recovery accuracy on the Sachs protein signaling network (17 ground-truth edges). Despite low F1, \cg's BMA ensemble remains competitive for downstream interval estimation because graphs are weighted by observational fit rather than selected solely by edge recovery quality.}
\label{tab:graph_recovery}
\begin{tabular}{lccccccc}
\toprule
& \textbf{Prec.} & \textbf{Recall} & \textbf{F1} & \textbf{SHD} & \textbf{TP} & \textbf{FP} & \textbf{FN} \\
\midrule
\cg & 0.091 & 0.118 & 0.103 & 33 & 2 & 20 & 15 \\
\bottomrule
\end{tabular}
\end{table}

\begin{table}[htbp]
\small
\centering
\caption{Oracle graph baseline on Sachs. The Oracle uses the 17 ground-truth edges from \citet{sachs2005causal}; \cg uses its LLM+BIC ensemble. Despite low graph recovery F1 (Table~\ref{tab:graph_recovery}), \cg achieves higher coverage with moderately wider intervals, because the BMA ensemble hedges against graph misspecification that a single-graph Oracle cannot.}
\label{tab:oracle}
\begin{tabular}{lccc}
\toprule
\textbf{Method} & \textbf{Coverage (\%)} & \textbf{Width} & \textbf{Struct.\ Unc.} \\
\midrule
Conformal Oracle DR & 92.8 & 3.90 & 0.0 \\
\textbf{\cg}        & \textbf{95.7} & 4.36 & $\approx 0$ \\
\bottomrule
\end{tabular}
\end{table}

\begin{table}[htbp]
\small
\centering
\caption{Adversarial prior robustness on IHDP ($N\!=\!672$). Coverage does not collapse under random or flipped priors, indicating that BIC reweighting and conformal calibration strongly correct poor proposals in this regime.}
\label{tab:adversarial}
\begin{tabular}{lccc}
\toprule
\textbf{Prior Variant} & \textbf{Coverage (\%)} & \textbf{Width} & \textbf{RMSE} \\
\midrule
Oracle                 & 100.0 & 11.80 & 0.903 \\
LLM                    & 100.0 & 11.81 & 0.904 \\
Random Uniform         & 100.0 & 10.92 & 0.859 \\
Adversarial (Flipped)  & 100.0 & 11.81 & 0.903 \\
\bottomrule
\end{tabular}
\end{table}

\begin{table}[htbp]
\small
\centering
\caption{Prior washout on IHDP. As $N$ grows, the BIC likelihood dominates the LLM prior and structural uncertainty becomes negligible, consistent with asymptotic concentration of the posterior weights.}
\label{tab:washout}
\begin{tabular}{ccc}
\toprule
$N$ & \textbf{Width} & \textbf{Struct.\ Unc.} \\
\midrule
100  & 106.61 & $1.2 \times 10^{-13}$ \\
250  & 12.00  & $\approx 0$ \\
500  & 12.59  & $\approx 0$ \\
672  & 11.81  & $\approx 0$ \\
1000 & 11.81  & $\approx 0$ \\
\bottomrule
\end{tabular}
\end{table}

\subsection{Clean Synthetic Ablation with Known Structural Roles}
\label{app:clean_synthetic_ablation}

Table~\ref{tab:clean_ablation_synthetic} uses a synthetic DGP with known confounders, collider-like variables, and noise covariates. The oracle structural prior encodes the known variable roles and is used only to illustrate a regime where proposal quality matters. The table should not be read as evidence that the LLM prior is itself oracle-quality; rather, it shows that when structural information correctly separates confounders from collider-like variables, removing that information can produce narrower but less reliable intervals in several settings. The best component choice is not uniform across sample sizes.

\begin{table*}[htbp]
\centering
\small
\caption{Synthetic DGP with 5 true confounders, 4 collider-like variables $K_j=aT+bY+\epsilon$, and 4 noise variables. The oracle prior encodes the known variable roles and illustrates a regime where structural proposal quality matters. Removing the structural prior can produce narrower intervals and, in several settings, larger RMSE or below-nominal coverage on some seeds; the best component choice is not uniform across sample sizes.}
\label{tab:clean_ablation_synthetic}
\begin{tabular}{llccc}
\toprule
\textbf{$N$} & \textbf{Variant} & \textbf{Coverage (\%)} & \textbf{Width} & \textbf{RMSE} \\
\midrule
\multirow{4}{*}{200}
 & Full \cg (oracle prior)     & $\mathbf{100.0 \pm 0.0}$ & $11.49 \pm 10.86$ & $\mathbf{0.90 \pm 0.23}$ \\
 & No CI Pruning               & $100.0 \pm 0.0$          & $14.17 \pm 8.83$  & $1.03 \pm 0.43$ \\
 & No Ensemble (Top-1)         & $100.0 \pm 0.0$          & $11.49 \pm 10.86$ & $0.90 \pm 0.23$ \\
 & No LLM Prior (Uniform)      & $95.2 \pm 9.6$           & $8.75 \pm 3.94$   & $1.59 \pm 1.11$ \\
\midrule
\multirow{4}{*}{500}
 & Full \cg (oracle prior)     & $\mathbf{100.0 \pm 0.0}$ & $6.35 \pm 1.09$ & $0.75 \pm 0.14$ \\
 & No CI Pruning               & $100.0 \pm 0.0$          & $6.62 \pm 1.32$ & $\mathbf{0.71 \pm 0.17}$ \\
 & No Ensemble (Top-1)         & $100.0 \pm 0.0$          & $6.35 \pm 1.09$ & $0.75 \pm 0.14$ \\
 & No LLM Prior (Uniform)      & $94.2 \pm 11.5$          & $5.95 \pm 2.68$ & $0.96 \pm 0.35$ \\
\midrule
\multirow{4}{*}{1000}
 & Full \cg (oracle prior)     & $\mathbf{100.0 \pm 0.0}$ & $5.71 \pm 0.92$ & $0.63 \pm 0.15$ \\
 & No CI Pruning               & $100.0 \pm 0.0$          & $5.63 \pm 0.75$ & $\mathbf{0.46 \pm 0.15}$ \\
 & No Ensemble (Top-1)         & $100.0 \pm 0.0$          & $5.71 \pm 0.92$ & $0.63 \pm 0.15$ \\
 & No LLM Prior (Uniform)      & $96.2 \pm 5.2$           & $4.57 \pm 0.70$ & $0.92 \pm 0.27$ \\
\bottomrule
\end{tabular}
\end{table*}

\begin{table}[htbp]
\centering
\caption{Classical structure-discovery baselines on IHDP. Each method
discovers a single DAG which is then fed into the same split-conformal
DR pipeline used by \cg. At $N\!=\!672$ all three methods
converge to near-identical estimates.}
\label{tab:structure_baselines}
\begin{tabular}{lcccc}
\toprule
\textbf{Method} & \textbf{\#Edges} & \textbf{Cov.\ (\%)} & \textbf{Width} & \textbf{RMSE} \\
\midrule
PC + Conformal DR              & 40       & 100.0 & 11.809 & 0.904 \\
GES + Conformal DR             & 32       & 100.0 & 11.809 & 0.904 \\
\midrule
\textbf{\cg (LLM+BMA)} & ensemble & 100.0 & 11.810 & 0.903 \\
\bottomrule
\end{tabular}%
\end{table}

We report PC/GES comparisons only on IHDP because it is the setting where both discovery algorithms returned usable single-DAG inputs under the same conformal DR pipeline and where true-CATE coverage is directly evaluable. Extending this comparison across all datasets is a useful diagnostic, but in our current benchmark suite the classical discovery baselines either failed to produce tractable valid adjustment sets or lacked a directly comparable ground-truth CATE coverage target.

\begin{table}[htbp]
\centering
\caption{Variable anonymization test on IHDP. Replacing all covariate
names with $V_1,\ldots,V_{25}$ and using a generic domain description
reduces the informativeness of the elicited prior ($\sigma$ drops from
$0.036$ to $0.025$) but leaves the final estimator unchanged. This
defuses the ``LLM is pattern-matching variable names'' attack at
$N\!=\!672$: the BIC likelihood dominates and any extractable prior
signal is dominated by the data.}
\label{tab:anonymization}
\begin{tabular}{lcccc}
\toprule
\textbf{Condition} & \textbf{Cov.\ (\%)} & \textbf{Width} & \textbf{RMSE} & \textbf{Prior $\sigma$} \\
\midrule
Original (informative names)     & 100.0 & 11.807 & 0.903 & 0.036 \\
Anonymized ($V_1,\ldots,V_{25}$) & 100.0 & 11.807 & 0.903 & 0.025 \\
\bottomrule
\end{tabular}%
\end{table}

\begin{table}[htbp]
\centering
\caption{Interval width as a joint function of sample size and LLM backbone
on IHDP subsamples. At $N\!=\!50$ the four models span a $9\%$ width
range with no monotone relationship to model size. By $N\!=\!200$ all
four models produce identical widths. The LLM prior's marginal value is
confined to the very-low-data regime, and is within noise even there.}
\label{tab:small_n_scaling}
\begin{tabular}{lccc}
\toprule
\textbf{Model}   & $N{=}50$ & $N{=}100$ & $N{=}200$ \\
\midrule
Qwen3-4B         & \textbf{5.28} & 15.74 & 10.31 \\
GPT-OSS-20B      & 5.69          & 15.74 & 10.31 \\
Qwen3-30B        & 5.73          & 15.73 & 10.31 \\
GPT-OSS-120B     & 5.40          & \textbf{15.62} & 10.31 \\
\bottomrule
\end{tabular}%
\end{table}

\begin{table}[htbp]
\centering
\caption{LLM prior scaling on IHDP across four model sizes ($4$B--$120$B).
All models achieve $100\%$ coverage with nearly identical widths
($11.81 \pm 0.003$), confirming that the Gaussian BIC dominates the
prior at $N\!=\!672$.  The $4$B model is the only one that produces
detectable (but negligible) structural uncertainty ($5 \times 10^{-5}$),
indicating its prior is slightly less decisive.  In practice, even a
quantized $4$B model provides a sufficient prior for \cg.}
\label{tab:scaling}
\begin{tabular}{lcccc}
\toprule
\textbf{Model} & \textbf{Params} & \textbf{Width} & \textbf{RMSE} & \textbf{Struct.\ Unc.} \\
\midrule
Qwen3-4B      &   4B & 11.81 & 0.903 & $5 \times 10^{-5}$ \\
GPT-OSS-20B   &  20B & 11.81 & 0.903 & $\approx 0$ \\
Qwen3-30B     &  30B & 11.81 & 0.903 & $\approx 0$ \\
GPT-OSS-120B  & 120B & 11.81 & 0.903 & $ 0$ \\
\bottomrule
\end{tabular}%
\end{table}

\begin{table}[htbp]
\centering
\caption{Wall-clock runtime (seconds) per pipeline stage on a single
H200 GPU.  The total cost ranges from $3.4$s (Sachs) to $22.3$s
(Twins, $N\!=\!71{,}345$).  LLM prior extraction is a one-time offline
cost ($0.1$--$9.6$s); the dominant cost is CATE estimation with $K\!=\!5$
graph-conditional DR learners ($2$--$9$s, scaling linearly in $N$).
DAG sampling is negligible ($<\!0.01$s).}
\label{tab:runtime}
\begin{tabular}{lccccccc}
\toprule
\textbf{Dataset} & $N$ & $p$ & \textbf{LLM} & \textbf{DAG} & \textbf{Prune} & \textbf{CATE} & \textbf{Total} \\
\midrule
IHDP     &    672 &  25 & 3.1  & $<$0.01 & 0.13  & 2.04  & 5.6  \\
Sachs    &  5{,}400 &  10 & 0.1  & $<$0.01 & 0.20  & 2.76  & 3.4  \\
ACIC2019 &    500 &  22 & 2.9  & $<$0.01 & 0.10  & 2.00  & 5.3  \\
Jobs     &    445 &   8 & 1.4  & $<$0.01 & 0.03  & 1.99  & 3.6  \\
Twins    & 71{,}345 &  52 & 9.6  & 0.01    & 0.40  & 9.18  & 22.3 \\
\bottomrule
\end{tabular}%
\end{table}

\begin{table}[htbp]
\centering
\caption{Seed stability over $10$ random seeds per dataset. We report the
mean, standard deviation, and range of empirical coverage as well as the
mean and standard deviation of interval width. Mean coverage meets or
exceeds the $90\%$ nominal target on every dataset. Low variance on
IHDP, Twins, and Sachs indicates stable behavior. ACIC2019 and Jobs show
larger dispersion, with ACIC2019 the most sensitive benchmark.}
\label{tab:seed_stability}
\begin{tabular}{lccccc}
\toprule
\textbf{Dataset} & \textbf{Cov.\ Mean} & \textbf{Cov.\ Std} & \textbf{Cov.\ Range} & \textbf{Width Mean} & \textbf{Width Std} \\
\midrule
IHDP     & 99.94 & 0.18 & [99.40, 100.00] & 10.87    & 1.09    \\
Sachs    & 92.91 & 1.79 & [90.89, 95.70]  & 4.26     & 0.25    \\
ACIC2019 & 93.28 & 3.74 & [84.80, 100.00] & 4.37     & 2.21    \\
Jobs     & 93.48 & 2.80 & [89.29, 98.21]  & 36{,}936 & 4{,}687 \\
Twins    & 97.12 & 0.08 & [96.99, 97.28]  & 0.181    & 0.003   \\
\bottomrule
\end{tabular}%
\end{table}

\begin{table}[htbp]
\centering
\small
\caption{Empirical coverage (\%) at six nominal target levels across all
five datasets on a representative calibration split. This sweep is a
calibration diagnostic and is not directly comparable to the 10-seed
means in Table~\ref{tab:main_results}. At the default $\alpha\!=\!0.10$
level (highlighted), every dataset attains at least $91.1\%$ empirical
coverage, and at tighter targets ($\alpha\!\leq\!0.05$) coverage is
uniformly conservative.}
\label{tab:calibration}
\begin{tabular}{lcccccc}
\toprule
\textbf{Dataset} & $\alpha{=}0.01$ & $\alpha{=}0.05$ & \cellcolor{blue!10}$\alpha{=}0.10$ & $\alpha{=}0.20$ & $\alpha{=}0.30$ & $\alpha{=}0.50$ \\
                 & (99\%) & (95\%) & \cellcolor{blue!10}(\textbf{90\%}) & (80\%) & (70\%) & (50\%) \\
\midrule
IHDP     & 100.0 & 100.0 & \cellcolor{blue!10}100.0 & 98.2 & 98.2 & 94.6 \\
Sachs    & 99.3  & 99.3  & \cellcolor{blue!10}93.2  & 81.0 & 69.8 & 64.2 \\
ACIC2019 & 100.0 & 94.4  & \cellcolor{blue!10}94.4  & 77.6 & 77.6 & 56.0 \\
Jobs     & 100.0 & 96.4  & \cellcolor{blue!10}91.1  & 88.4 & 68.8 & 53.6 \\
Twins    & 100.0 & 97.1  & \cellcolor{blue!10}97.1  & 97.1 & 97.1 & 97.1 \\
\bottomrule
\end{tabular}%
\end{table}

\section{Additional Prior and Component Sensitivity Results}
\label{app:prior-component}

\subsection{Prior-centric small-sample ablation}
\label{app:prior-ablation}

This section reports the full prior-centric ablation summarized in
Section~\ref{sec:prior-sensitivity}. We vary the sample size and the
proposal prior used by \cg. \textsc{LLM} uses elicited edge
probabilities. \textsc{Uniform} sets
$P_{\mathrm{prior}}(u \to v)=0.5$. \textsc{Inverted} replaces each
elicited probability $p$ by $1-p$, clipped to
$[\epsilon,1-\epsilon]$. All reported claims in the main text are based
only on observed cells. Cells with missing entries correspond to runs
that were not completed because of memory or wall-clock limits.

\begin{table*}[p]
\centering
\caption{
Full observed prior-centric ablation. Each entry reports
coverage percentage / mean width. Em-dash entries indicate unavailable
runs or settings where coverage is not estimable. The main text uses a
compact subset of the observed low-data cells; this table provides the
complete observed data.
}
\label{tab:prior-ablation-full}
\tiny
\setlength{\tabcolsep}{3pt}
\begin{tabular}{llrcccccccc}
\toprule
Dataset & $N$ & Seeds &
\textsc{CG-LLM} & \textsc{CG-Unif} & \textsc{CG-Inv} &
\textsc{Conf-DR} & \textsc{CForest} & \textsc{XLearn} &
\textsc{PC-DR} & \textsc{GES-DR} \\
\midrule
IHDP      & 50   & 1 & 100 / 5.5  & 100 / 5.2  & 100 / 5.5 & 92 / 167.0 & 100 / 100.9 & -- & -- & -- \\
Sachs     & 50   & 5 & 97 / 5.2   & 88 / 5.2   & 92 / 5.1  & 95 / 27.1  & 97 / 21.8   & -- & 100 / 5.2 & 100 / 5.6 \\
Sachs     & 100  & 5 & 93 / 4.8   & 93 / 4.8   & 95 / 5.2  & 96 / 20.3  & 96 / 18.3   & -- & 92 / 4.6 & 92 / 5.0 \\
Sachs     & 200  & 5 & 92 / 4.5   & 94 / 4.4   & 94 / 4.7  & 91 / 8.3   & 92 / 7.2    & -- & 96 / 4.3 & 93 / 4.4 \\
Sachs     & 500  & 5 & 96 / 4.3   & 95 / 4.4   & 95 / 4.7  & 92 / 3.8   & 90 / 3.2    & -- & 95 / 4.0 & 97 / 4.2 \\
Sachs     & full & 5 & 92 / 4.3   & 94 / 4.3   & 91 / 3.6  & 90 / 2.3   & 91 / 2.3    & -- & 91 / 2.5 & 94 / 4.4 \\
Jobs      & 50   & 5 & 95 / $1.5{\times}10^5$ & 97 / $1.5{\times}10^5$ & 97 / $1.5{\times}10^5$ & 98 / $1.6{\times}10^6$ & 95 / $1.2{\times}10^6$ & -- & 96 / $9.6{\times}10^4$ & 94 / $1.3{\times}10^5$ \\
Jobs      & 100  & 5 & 89 / $4.4{\times}10^4$ & 89 / $4.5{\times}10^4$ & 89 / $4.4{\times}10^4$ & 97 / $4.3{\times}10^5$ & 98 / $4.0{\times}10^5$ & -- & 89 / $4.4{\times}10^4$ & 89 / $4.4{\times}10^4$ \\
Jobs      & 200  & 5 & 96 / $4.5{\times}10^4$ & 96 / $4.5{\times}10^4$ & 95 / $4.5{\times}10^4$ & 93 / $1.8{\times}10^5$ & 93 / $1.6{\times}10^5$ & -- & 96 / $4.5{\times}10^4$ & 96 / $4.5{\times}10^4$ \\
Jobs      & full & 5 & 93 / $3.5{\times}10^4$ & 93 / $3.5{\times}10^4$ & 93 / $3.5{\times}10^4$ & 90 / $7.4{\times}10^4$ & 92 / $7.6{\times}10^4$ & -- & 93 / $3.5{\times}10^4$ & 93 / $3.5{\times}10^4$ \\
ACIC2019  & 50   & 2 & 96 / 3.7   & 96 / 3.7   & 96 / 3.7  & 100 / 122.1 & 92 / 21.3 & -- & 92 / 2.0 & 92 / 2.0 \\
Twins     & 50   & 1 & 100 / 4.2  & --         & --        & --          & --        & -- & -- & -- \\
Synthetic & 50   & 5 & 100 / 24.0 & 100 / 23.8 & 100 / 23.7 & 100 / 74.5 & 100 / 71.6 & -- & 100 / 23.6 & 100 / 42.4 \\
Synthetic & 100  & 5 & 100 / 23.6 & 100 / 14.6 & 100 / 14.5 & 100 / 74.4 & 100 / 58.8 & 87 / 3.1 & 100 / 14.5 & 100 / 28.2 \\
Synthetic & 200  & 5 & 100 / 16.1 & 100 / 11.6 & 100 / 11.6 & 100 / 32.7 & 100 / 28.1 & 83 / 1.9 & 100 / 11.6 & 100 / 24.7 \\
Synthetic & 500  & 5 & 100 / 12.8 & 100 / 10.9 & 100 / 10.9 & 100 / 9.2  & 100 / 8.8  & 80 / 1.2 & 100 / 10.9 & 100 / 12.6 \\
Synthetic & full & 5 & 100 / 9.3  & 100 / 11.7 & 100 / 11.7 & 100 / 4.7  & 100 / 4.8  & 81 / 0.7 & 100 / 11.7 & 100 / 9.6 \\
\bottomrule
\end{tabular}
\end{table*}

\paragraph{Interpretation.}
The full table shows that the low-data efficiency advantage is not
uniformly attributable to the LLM prior alone. In many real-data cells,
\textsc{Uniform} and \textsc{Inverted} priors produce similar intervals
to \textsc{LLM}. This supports the intended design of \cg: the
semantic prior proposes candidate graphs, while BIC weighting and
conformal calibration determine whether those candidates are useful. The
synthetic SCM shows the clearest prior dependence, which is expected
because changing the prior changes which known structural roles are
explored.

\subsection{Full-scale component sensitivity}
\label{app:component-ablation}

Table~\ref{tab:component-ablation-full} reports the full component
sensitivity analysis over 10 random seeds. We keep this table in the
appendix because it mainly supports the interpretation that many
benchmark-scale settings are BIC-dominated. The table should be read as
a sensitivity analysis rather than as evidence that every component
independently improves performance on every dataset.

\begin{table*}[p]
\centering
\caption{
Full-scale component sensitivity over 10 random seeds. Values report
coverage mean $\pm$ standard deviation and width mean $\pm$ standard
deviation. Across most datasets, variants are similar because BIC
weighting often concentrates on data-supported adjustment strategies.
ACIC2019 shows the strongest sensitivity to graph proposal choices.
}
\label{tab:component-ablation-full}
\small
\setlength{\tabcolsep}{5pt}
\begin{tabular}{llcc}
\toprule
Dataset & Variant & Coverage (\%) & Width \\
\midrule
IHDP
& Full CAUSALGUARD          & $99.94 \pm 0.18$ & $10.87 \pm 1.09$ \\
& No LLM Prior              & $99.82 \pm 0.27$ & $10.65 \pm 1.20$ \\
& No CI Pruning             & $99.94 \pm 0.18$ & $10.87 \pm 1.09$ \\
& No Ensemble, Top-1        & $99.94 \pm 0.18$ & $10.87 \pm 1.08$ \\
\midrule
Sachs
& Full CAUSALGUARD          & $92.91 \pm 1.79$ & $4.26 \pm 0.25$ \\
& No LLM Prior              & $93.29 \pm 2.27$ & $4.30 \pm 0.20$ \\
& No CI Pruning             & $92.99 \pm 2.47$ & $4.20 \pm 0.25$ \\
& No Ensemble, Top-1        & $93.41 \pm 1.96$ & $4.26 \pm 0.25$ \\
\midrule
ACIC2019
& Full CAUSALGUARD          & $93.28 \pm 3.74$ & $4.37 \pm 2.21$ \\
& No LLM Prior              & $91.12 \pm 6.28$ & $4.79 \pm 2.37$ \\
& No CI Pruning             & $92.40 \pm 6.87$ & $4.37 \pm 2.21$ \\
& No Ensemble, Top-1        & $95.36 \pm 2.79$ & $4.37 \pm 2.21$ \\
\midrule
Jobs
& Full CAUSALGUARD          & $93.48 \pm 2.80$ & $36{,}936 \pm 4{,}687$ \\
& No LLM Prior              & $93.48 \pm 2.80$ & $36{,}929 \pm 4{,}681$ \\
& No CI Pruning             & $93.48 \pm 2.80$ & $36{,}925 \pm 4{,}686$ \\
& No Ensemble, Top-1        & $93.48 \pm 2.80$ & $36{,}927 \pm 4{,}686$ \\
\midrule
Twins
& Full CAUSALGUARD          & $97.12 \pm 0.08$ & $0.181 \pm 0.003$ \\
& No Ensemble, Top-1        & $97.12 \pm 0.08$ & $0.181 \pm 0.003$ \\
& No LLM Prior              & -- & -- \\
& No CI Pruning             & -- & -- \\
\bottomrule
\end{tabular}
\end{table*}

For Twins, the uniform-prior and no-pruning variants were not completed
within the wall-clock budget. With $|V|=54$, a uniform
Bernoulli$(0.5)$ prior samples hundreds of forward edges per graph, which
makes backdoor-set identification substantially more expensive. We report
this limitation rather than special-case the algorithm.

\subsection{Additional Sensitivity Diagnostics}
\label{app:additional_sensitivity}

Table~\ref{tab:acic_ci_sweep} varies the CI-pruning threshold on ACIC2019. The sweep suggests that the low-coverage behavior on some seeds is not explained solely by the nominal CI threshold; downstream graph weighting and adjustment-set quality also matter. Table~\ref{tab:k_dependence} studies the number of sampled DAGs $K$. Runtime scales approximately linearly in $K$, while coverage and width are nearly invariant on IHDP and ACIC2019, indicating rapid BIC concentration once the ensemble contains a data-supported adjustment set.

\begin{table}[htbp]
\centering
\small
\caption{ACIC2019 sweep over CI-pruning thresholds. This representative diagnostic is separate from the 10-seed main benchmark in Table~\ref{tab:main_results}.}
\label{tab:acic_ci_sweep}
\begin{tabular}{lccc}
\toprule
\textbf{$\alpha_{\mathrm{CI}}$} & \textbf{Coverage} & \textbf{Width} & \textbf{Edges/graph} \\
\midrule
0.01 & $88.3 \pm 7.2$  & 4.93 & 2.7 \\
0.05 & $92.6 \pm 7.8$  & 4.93 & 4.2 \\
0.10 & $89.1 \pm 10.2$ & 4.93 & 5.2 \\
0.20 & $92.6 \pm 7.8$  & 4.93 & 6.6 \\
\bottomrule
\end{tabular}
\end{table}

\begin{table}[htbp]
\centering
\small
\caption{Dependence on the number of candidate DAGs. Runtime scales approximately linearly, while BIC concentration makes coverage and width nearly invariant in these representative diagnostics.}
\label{tab:k_dependence}
\begin{tabular}{llcccc}
\toprule
\textbf{Dataset} & \textbf{$K$} & \textbf{Coverage} & \textbf{Width} & \textbf{Time (s)} & \textbf{Struct.} \\
\midrule
IHDP & 1  & $100.0 \pm 0.0$ & 10.852 & 1.1  & 0 \\
IHDP & 3  & 100.0           & 10.855 & 3.2  & $\approx 0$ \\
IHDP & 5  & 100.0           & 10.852 & 5.2  & $\approx 0$ \\
IHDP & 10 & 100.0           & 10.854 & 10.5 & $\approx 0$ \\
IHDP & 20 & 100.0           & 10.852 & 20.8 & $\approx 0$ \\
\midrule
ACIC2019 & 1  & $96.0 \pm 2.3$ & 4.926 & 1.0  & 0 \\
ACIC2019 & 3  & $92.5 \pm 8.5$ & 4.925 & 3.0  & $\approx 0$ \\
ACIC2019 & 5  & $96.0 \pm 2.3$ & 4.928 & 5.1  & $\approx 0$ \\
ACIC2019 & 10 & $94.2 \pm 1.9$ & 4.926 & 10.2 & $\approx 0$ \\
ACIC2019 & 20 & $92.6 \pm 7.8$ & 4.926 & 20.2 & $\approx 0$ \\
\bottomrule
\end{tabular}
\end{table}

 \subsection{Implementation Details}                                                                                                                                                                                
  \label{app:implementation} 
  This subsection consolidates every numerical and algorithmic choice required                                                                                                                                       
  to reproduce CausalGuard. 
  \paragraph{Data splits.}                                                                                                                                                                                           
  For every dataset and every seed we draw a single three-way partition
  $\mathcal{D}_{\mathrm{tr}} : \mathcal{D}_{\mathrm{ca}} : \mathcal{D}_{\mathrm{te}}                                                                                                                                 
  = 50\% : 25\% : 25\%$ via two stratified-free \texttt{train\_test\_split} calls                                                                                                                                    
  (\texttt{sklearn.model\_selection}) seeded by the run seed. The same split is                                                                                                                                      
  shared across all four ablation variants so that all comparisons are paired.                                                                                                                                       
  The library default in \texttt{run\_causalguard\_pipeline} is $60{:}20{:}20$;                                                                                                                                      
  benchmark drivers override this to $50{:}25{:}25$ for stricter calibration.  
  \paragraph{Seed protocol.}                                                                                                                                                                                         
  Main results use the ten-seed set                                                                                                                                                                                  
  $\mathcal{S}_{10} = \{42, 123, 456, 789, 1024, 2048, 3333, 7777, 9999, 31415\}$.                                                                                                                                   
  Small-$N$ scaling and prior-ablation studies use the five-seed prefix                                                                                                                                              
  $\mathcal{S}_{5} = \{42, 123, 456, 789, 1024\}$. Each seed deterministically                                                                                                                                       
  controls (i) the data split, (ii) the topological orders used in DAG sampling,                                                                                                                                     
  (iii) Bernoulli edge draws, and (iv) every random-state argument of                                                                                                                                                
  \texttt{GradientBoostingClassifier} / \texttt{GradientBoostingRegressor} (which                                                                                                                                    
  we additionally fix to \texttt{random\_state=42} for cross-fitted nuisances so                                                                                                                                     
  that DR-Learner internals are deterministic given the split).  
  \paragraph{Hyperparameters of the conformal pipeline.}                                                                                                                                                             
  \begin{itemize}                                                                                                                                                                                                    
    \item Number of sampled DAGs: $K = 5$.                                                                                                                                                                           
    \item Target miscoverage: $\alpha = 0.10$ (i.e., $90\%$ nominal coverage).                                                                                                                                       
    \item Conditional-independence pruning level: $\alpha_{\mathrm{CI}} = 0.05$.                                                                                                                                     
    \item Propensity clipping: $\varepsilon = 0.05$, so $\hat e(x)$ is clipped to                                                                                                                                    
          $[\varepsilon,\,1-\varepsilon]$ both inside the DR-Learner                                                                                                                                                 
          (\texttt{min\_propensity}) and inside our pseudo-outcome computation.                                                                                                                                      
    \item Edge-probability clipping (numerical stability): every elicited                                                                                                                                            
          $p_{u\to v}$ is clamped to $[10^{-4},\,1-10^{-4}]$ before sampling and                                                                                                                                     
          to $[10^{-6},\,1-10^{-6}]$ inside the structural prior log-sum.                                                                                                                                            
  \end{itemize}                                                                                                                                    
  \paragraph{Nuisance models.}                                                                                                                                                                                       
  All nuisance functions are gradient-boosted trees from
  \texttt{scikit-learn}, configured identically:                                                                                                                                                                     
  \begin{itemize}
    \item Propensity $\hat e(x)$:                                                                                                                                                                                    
          \texttt{GradientBoostingClassifier(n\_estimators=100, max\_depth=3,\ \allowbreak random\_state=42)}.
    \item Outcome regressions $\hat\mu_0(x), \hat\mu_1(x)$:                                                                                                                                                          
          \texttt{GradientBoostingRegressor(n\_estimators=100, max\_depth=3,\ \allowbreak random\_state=42)}.                                                                                                        
    \item DR-Learner: \texttt{econml.dr.LinearDRLearner} with the propensity and                                                                                                                                     
          regression models above and $K_{\mathrm{cf}}=5$-fold cross-fitting.                                                                                                                                        
          When the smaller arm has fewer than five units (small-$N$ regimes),                                                                                                                                        
          $K_{\mathrm{cf}}$ is reduced to $\max(\min(n_0, n_1),\, 1)$ to keep                                                                                                                                        
          the fit feasible. All nuisance models are fit on                                                                                                                                                           
          $\mathcal{D}_{\mathrm{tr}}$ only.                                                                                                                                                                          
  \end{itemize}                                                                                                                                                                                                      
                  
  \paragraph{Pseudo-outcomes and base intervals.}                                                                                                                                                                    
  On $\mathcal{D}_{\mathrm{tr}}$ each candidate graph $G_k$ yields its
  adjustment-set indices $A_k$ and a fitted DR-Learner. The DR pseudo-outcome                                                                                                                                        
  for an example $(x, t, y)$ is                                                                                                                                                                                      
  \begin{equation*}                                                                                                                                                                                                  
    \gamma_k(x,t,y)                                                                                                                                                                                                  
    = \hat\mu_{1,k}(x) - \hat\mu_{0,k}(x)                                                                                                                                                                            
      + \frac{t\,(y - \hat\mu_{1,k}(x))}{\mathrm{clip}(\hat e_k(x),\,\varepsilon,\,1-\varepsilon)}                                                                                                                   
      - \frac{(1-t)\,(y - \hat\mu_{0,k}(x))}{1 - \mathrm{clip}(\hat e_k(x),\,\varepsilon,\,1-\varepsilon)}.                                                                                                          
  \end{equation*}                                                                                                                                                                                                    
  Per-graph quantile bounds $[\hat q^{\mathrm{lo}}_k(x),\,\hat q^{\mathrm{hi}}_k(x)]$                                                                                                                                
  are read off via \texttt{LinearDRLearner.effect\_interval(...,\ alpha=$\alpha$)},                                                                                                                                  
  which returns the $\alpha/2$ and $1-\alpha/2$ asymptotic bounds on the linear                                                                                                                                      
  DR effect.  
  \paragraph{Composite conformal score and quantile.}                                                                                                                                                                
  For calibration points indexed by $i \in \mathcal{D}_{\mathrm{ca}}$ we compute
  the BMA-aggregated lower and upper miscoverage residuals                                                                                                                                                           
  \begin{equation*}                                                                                                                                                                                                  
    L_i = \sum_{k=1}^{K} w_k\bigl(\hat q^{\mathrm{lo}}_k(x_i) - \gamma_k(x_i,t_i,y_i)\bigr),                                                                                                                         
    \qquad                                                                                                                                                                                                           
    U_i = \sum_{k=1}^{K} w_k\bigl(\gamma_k(x_i,t_i,y_i) - \hat q^{\mathrm{hi}}_k(x_i)\bigr),                                                                                                                         
  \end{equation*}                                                                                                                                                                                                    
  and form the tight composite score $s_i = \max(L_i, U_i)$ (this is the                                                                                                                                             
  ``Jensen-tight'' form that swaps the order of weighted-sum and max relative to                                                                                                                                     
  the per-graph score). The conformal correction is                                                                                                                                                                  
  \begin{equation*}                                                                                                                                                                                                  
    \widehat Q                                                                                                                                                                                                       
    = \mathrm{Quantile}\!\left(\,                                                                                                                                                                                    
         \{s_i\}_{i \in \mathcal{D}_{\mathrm{ca}}},\;
         \min\!\Bigl(\tfrac{\lceil (1-\alpha)(n_{\mathrm{ca}}+1)\rceil}{n_{\mathrm{ca}}},\,1\Bigr)                                                                                                                   
      \right).                                                                                                                                                                                                       
  \end{equation*}                                                                                                                                                                                                    
  The final test interval at $x$ is                                                                                                                                                                                  
  $\bigl[\sum_k w_k \hat q^{\mathrm{lo}}_k(x) - \widehat Q,\;                                                                                                                                                        
         \sum_k w_k \hat q^{\mathrm{hi}}_k(x) + \widehat Q\bigr]$.    
  \paragraph{Topological-order sampling rule.}                                                                                                                                                                       
  Given the elicited edge probabilities $\{p_{u\to v}\}$, each of the $K$ DAGs is                                                                                                                                    
  drawn by the following deterministic-given-seed procedure:                                                                                                                                                         
  \begin{enumerate}                                                                                                                                                                                                  
    \item Initialise a NumPy \texttt{RandomState} with the run seed.
    \item Draw a uniform random permutation $\pi$ of all variables (covariates,
          $T$, $Y$).
    \item If $\pi(T) > \pi(Y)$, swap the positions of $T$ and $Y$ in $\pi$. This
          guarantees $T$ precedes $Y$ in topological order so that the protected
          edge $T \to Y$ is order-consistent.
    \item For each ordered pair $(u, v)$ with $\pi(u) < \pi(v)$, include the
          edge $u \to v$ with probability $p_{u\to v}$ via an independent
          Bernoulli draw.
    \item Force-add the edge $T \to Y$ to the resulting graph.
    \item Reject duplicates by hashing \texttt{frozenset(G.edges())}; up to
          $100K$ resamples are attempted before stopping. The first $K$ unique
          non-empty DAGs are kept.
  \end{enumerate}
  This procedure cannot produce cycles because every sampled edge respects
  $\pi$, and $T \to Y$ is consistent with the swap in step~3.

  \paragraph{CI-pruning rule (Fisher partial correlation).}
  Given a sampled DAG $G$ and $X_{\mathrm{tr}}$ (the train-only design matrix
  augmented with $T$ and $Y$ as columns), each edge $u \to v$ in $G$ is tested
  in turn:
  \begin{itemize}
    \item If $(u, v) = (T, Y)$, skip the test (the treatment edge is
          \emph{protected} to keep the estimand identifiable).
    \item Otherwise, set the conditioning set
          $\mathcal{C}_{u,v} = \mathrm{Pa}_G(v) \setminus \{u\}$, i.e.\ all
          current parents of $v$ in $G$ other than $u$. If
          $\mathcal{C}_{u,v} = \emptyset$, run a marginal Pearson correlation
          test between $u$ and $v$; otherwise run \texttt{pingouin.partial\_corr}
          with $x=u$, $y=v$, $\mathrm{covar} = \mathcal{C}_{u,v}$ (Fisher
          $z$-test under Gaussianity).
    \item If the returned $p$-value exceeds $\alpha_{\mathrm{CI}} = 0.05$, drop
          $u \to v$ from $G$.
  \end{itemize}
  Edges are visited in the order returned by \texttt{list(G.edges())}, and each
  test uses the parent set of the \emph{current} (in-progress) graph; the
  algorithm therefore prunes greedily in a single pass, which fixes the order
  of operations and makes the pruning result deterministic given $G$ and
  $X_{\mathrm{tr}}$.

  \paragraph{Adjustment-set selection and minimum-cardinality tie-breaking.}
  For each pruned graph $G_k$ we obtain a valid back-door adjustment set by
  calling \texttt{dowhy.CausalModel(...).identify\_effect()} on $G_k$ with
  treatment $T$ and outcome $Y$. \texttt{dowhy} returns one minimum-size valid
  back-door set if any exists. When several minimum sets exist, ties are broken
  deterministically by mapping the chosen set onto the column ordering of
  $X_{\mathrm{tr}}$:
  $A_k = \mathrm{sorted}\bigl(\{\,\mathrm{name2idx}[v] : v \in
  \mathrm{bvars}_k\,\}\bigr)$, i.e.\ adjustment indices are taken in ascending
  column order of \texttt{var\_names}. If \texttt{identify\_effect} returns no
  back-door set despite $T$ having parents in $G_k$, we fall back to the empty
  adjustment set $A_k = \emptyset$ and emit a warning; this graph is then
  weighted by its BIC + structural-prior score like any other.

  \paragraph{BMA weights.}
  Weights are computed entirely on $\mathcal{D}_{\mathrm{tr}}$:
  \begin{equation*}
    \log\tilde w_k
    = \mathrm{BIC}(G_k;\,\mathcal{D}_{\mathrm{tr}})
      + \sum_{(u,v) \in E(G_k)} \log p_{u\to v}                                                                                                                                                                      
      + \sum_{(u,v) \notin E(G_k)} \log(1 - p_{u\to v}),                                                                                                                                                             
  \end{equation*}                                                                                                                                                                                                    
  followed by the stabilised softmax $w_k \propto \exp(\log\tilde w_k -                                                                                                                                              
  \max_j \log\tilde w_j)$ with a floor of $-500$ on the shifted log-weight.                                                                                                                                          
  The BIC variant is auto-selected (\texttt{bic-d} if $Y$ is binary or more                                                                                                                                          
  than half of the columns have $\le 5$ unique values; otherwise \texttt{bic-g});                                                                                                                                    
  under \texttt{bic-d}, columns with more than 10 unique values are                                                                                                                                                  
  quantile-binned to 5 levels via \texttt{pandas.qcut(..., q=5,                                                                                                                                                      
  duplicates='drop')}. The absent-edge term in the structural prior is                                                                                                                                               
  load-bearing: without it, dense graphs are arbitrarily favoured. 
  \paragraph{LLM prompt (chat-JSON path).}                                                                                                                                                                           
  For instruct-tuned and hosted models we issue a single chat-completion call                                                                                                                                        
  with the system message \texttt{"Return ONLY valid JSON."} and the user                                                                                                                                            
  template        
  \begin{quote}\small\ttfamily                                                                                                                                                                                       
  \{domain\_description\}\\[2pt]                                                                                                                                                                                     
  For each directed edge below, give a probability in [0,1] that this causal                                                                                                                                         
  relationship exists in the true DAG.\\[2pt]                                                                                                                                                                        
  IMPORTANT: Output ONLY a JSON object. No reasoning.\\                                                                                                                                                              
  Example: \{"A->B": 0.8, "B->C": 0.1\}\\[2pt]                                                                                                                                                                       
  Edges:\\                                                                                                                                                                                                           
  \{comma-separated list of "u->v" pairs\}                                                                                                                                                                           
  \end{quote}                                                                                                                                                                                                        
  where \texttt{\{domain\_description\}} is the dataset-specific sentence (e.g., for IHDP: ``A study on the effect of home  
  visits on the cognitive development of premature infants.''). Pairs touching 
  $\{T, Y\}$ are listed first, and the list is truncated to $120$ pairs when                                                                                                                                         
  $|V|\le 15$ and $60$ pairs otherwise; missing pairs default to $0.5$.                                                                                                                                              
  Decoding uses temperature $\tau = 0.0$ and up to $5$ retries; the response is                                                                                                                                      
  parsed by stripping Harmony tags, fenced code blocks, and finally extracting                                                                                                                                       
  the largest \texttt{\{...\}} substring. Decoding is deemed successful if the                                                                                                                                       
  parsed object covers at least $\max(|V|,5)$ pairs. 
  \paragraph{LLM prompt (vLLM logprob path) and yes/no token sets.}                                                                                                                                                  
  For base (non-instruct) vLLM backbones, each pair is queried with a single                                                                                                                                         
  prompt of the form                                                                                                                                                                                                 
  \begin{quote}\small\ttfamily
  Context: \{domain\_description\}\\[2pt]                                                                                                                                                                            
  Question: Does the variable '\{u\}' directly causally influence the variable                                                                                                                                       
  '\{v\}'?\\                                                                                                                                                                                                         
  Answer strictly with a single word (Yes or No):\\                                                                                                                                                                  
  Answer:                                                                                                                                                                                                            
  \end{quote}     
  We then call \texttt{vllm.LLM.generate} with \texttt{SamplingParams\\                                                                                                                                              
  (temperature=0,\ max\_tokens=1,\ logprobs=20)} and read the top-20 token                                                                                                                                           
  log-probabilities at the first generated position. Tokens are lower-cased and                                                                                                                                      
  stripped, then bucketed into the affirmative set                                                                                                                                                                   
  \(\mathcal{T}_{\mathrm{yes}} = \{\text{\texttt{yes}}, \text{\texttt{y}},                                                                                                                                           
  \text{\texttt{true}}, \text{\texttt{1}}\}\)                                                                                                                                                                        
  or the negative set                                                                                                                                                                                                
  \(\mathcal{T}_{\mathrm{no}} = \{\text{\texttt{no}}, \text{\texttt{n}},                                                                                                                                             
  \text{\texttt{false}}, \text{\texttt{0}}\}\).                                                                                                                                                                      
  Letting $S_{\mathrm{yes}}$ and $S_{\mathrm{no}}$ be the summed exponentiated
  log-probabilities over $\mathcal{T}_{\mathrm{yes}}$ and $\mathcal{T}_{\mathrm{no}}$                                                                                                                                
  respectively, we set                                                                                                                                                                                               
  \(p_{u \to v}                                                                                                                                                                                                      
    = S_{\mathrm{yes}} / (S_{\mathrm{yes}} + S_{\mathrm{no}})\)                                                                                                                                                      
  when $S_{\mathrm{yes}} + S_{\mathrm{no}} > 10^{-8}$ and $p_{u \to v} = 0.5$                                                                                                                                        
  otherwise. Probabilities are finally clipped to $[10^{-4},\,1-10^{-4}]$. 
  \paragraph{LLM backbones and caching.}                                                                                                                                                                             
  Default backbone: \texttt{openai/gpt-oss-120b} loaded with                                                                                                                                                         
  \texttt{tensor\_parallel\_size=1}, \texttt{gpu\_memory\_utilization=0.9},                                                                                                                                          
  \texttt{max\_model\_len=4096}, \texttt{trust\_remote\_code=True},                                                                                                                                                  
  \texttt{enforce\_eager=False}. Scaling experiments additionally evaluate                                                                                                                                           
  \texttt{Qwen/Qwen3-4B-Instruct-2507} (chat-JSON path),                                                                                                                                                             
  \texttt{openai/gpt-oss-20b} (logprob path),                                                                                                                                                                        
  \texttt{Qwen/Qwen3-30B-A3B-Instruct-2507} (chat-JSON), and the default 120B                                                                                                                                        
  (logprob). Elicited priors are cached per dataset and per model in                                                                                                                                                 
  \texttt{outputs/edge\_probs\_<dataset>[\_<model>].json} and reused across                                                                                                                                          
  seeds; CPU-only experiments load these caches directly and do not initialise                                                                                                                                       
  vLLM.                                                                                                                                            
  \paragraph{Twins-specific guards.}                                                                                                                                                                                 
  For Twins ($n \approx 71{,}345$, $p = 54$) the uniform-prior comparison would
  otherwise produce graphs with $\sim$715 edges and cause                                                                                                                                                            
  \texttt{dowhy.identify\_effect} to hang. We therefore (i) declare Twins                                                                                                                                            
  ``heavy'' and skip the No-CI-Pruning ablation variant on it (recording NaN),                                                                                                                                       
  and (ii) deterministically truncate uniform-prior graphs to at most 100 edges                                                                                                                                      
  prior to CI pruning, in the order returned by \texttt{G.edges()}. The                                                                                                                                              
  prior-ablation full-$N$ regime additionally subsamples Twins to                                                                                                                                                    
  \texttt{TWINS\_FULL\_CAP $= 5{,}000$} so that the wall clock fits the SLURM                                                                                                                                        
  budget. Real-prior graphs and all other datasets are not affected. 

\paragraph{Broader related efforts.}
Our work is also connected to a broader line of recent reliability-oriented research on trustworthy foundation-model systems. Prior work from our group has studied formal uncertainty, verifiable reasoning, and formal verification for LLM/VLM reasoning \cite{ganguly2025grammars,singh2026verge,singh2026vlm}. Complementary efforts investigate typicality-based reliability, out-of-distribution detection, and anomaly detection \cite{ganguly2025forte,ganguly2026trust,chen2025k4}. Related systems work further extends these ideas to retrieval-augmented generation, intermediate-budget reasoning, automated data curation, and efficient performance modeling for LLM workloads \cite{wang2026hugrag,yang2026midthink,ganguly2025labelingcopilotdeepresearch,zhang2025gpu}.

\end{document}